%% file: main_arxiv.tex
\documentclass[11pt]{article}

\usepackage[margin=1in]{geometry}
\usepackage[utf8]{inputenc} %
\usepackage[T1]{fontenc}    %
\usepackage[dvipsnames]{xcolor}
\definecolor{Bea_blue}{RGB}{71,127,124}
\definecolor{Bea_red}{RGB}{210, 77, 4}

\usepackage[colorlinks=true, linkcolor=Bea_red, citecolor=Bea_blue, linktoc=page]{hyperref}
\usepackage{url}            %
\usepackage{booktabs}       %
\usepackage{amsfonts}       %
\usepackage{nicefrac}       %
\usepackage{microtype}      %
\usepackage{xcolor}         %
\usepackage[numbers]{natbib}
\usepackage{Package}
\usepackage{NewCommand}

\title{Posterior Sampling with Denoising Oracles via Tilted Transport}

\author[1,2,3]{Joan Bruna}
\author[3]{Jiequn Han}
\affil[1]{Courant Institute of Mathematical Sciences, New York University}
\affil[2]{Center for Data Science, New York University}
\affil[3]{Flatiron Institute}

\begin{document}

\maketitle

\begin{abstract}

Score-based diffusion models have significantly advanced high-dimensional data generation across various domains, by learning a denoising oracle (or score) from datasets. From a Bayesian perspective, they offer a realistic modeling of data priors and facilitate solving inverse problems through posterior sampling.
Although many heuristic methods have been developed recently for this purpose, they lack the quantitative guarantees needed in many scientific applications. 

In this work, we introduce the 
\textit{tilted transport} technique, which leverages the quadratic structure of the log-likelihood in linear inverse problems in combination with the prior denoising oracle to  transform the original posterior sampling problem into a new `boosted' posterior that is provably easier to sample from. We quantify the conditions under which this boosted posterior is strongly log-concave, highlighting the dependencies on the condition number of the measurement matrix and the signal-to-noise ratio. The resulting posterior sampling scheme is shown to reach the computational threshold predicted for sampling Ising models \cite{doi:10.1137/1.9781611977912.180} with a direct analysis, and is further validated on high-dimensional Gaussian mixture models and scalar field $\varphi^4$ models. 
\end{abstract}

\tableofcontents

\section{Introduction}
Inverse problems consist in reconstructing a signal of interest from noisy measurements. As such, they are a central object of study across many scientific domains, including signal processing, imaging, astrophysics or computational biology. 
In the common settings where the measurement information is limited, a reliable solution for these problems usually depends on prior knowledge of the data. One popular approach is to choose a regularizer that utilizes data properties such as smoothness or sparseness, and then solve a regularized optimization problem to obtain \textit{a point estimate} of the original data. However, this approach often struggles with selecting an appropriate regularizer and might be unstable in the presence of large measurement noise. A more robust approach takes a statistical formulation and seeks to sample the \textit{posterior distribution} of data based on Bayes's theorem, which allows for uncertainty quantification in the reconstructed data by leveraging a model for the prior data distribution. 

While accurate models for high-dimensional distributions are notoriously complex to estimate, the resurgence of deep neural networks has provided unprecedented capabilities for modeling complex data distributions in certain high-dimensional regimes. Specifically, score-based diffusion models \cite{sohl2015deep,ho2020denoising,song2020score} have achieved remarkable empirical success in generating high-dimensional data across various domains, including images, video, text, and audio. These models implicitly parameterize data distributions through an iterative denoising process that builds up data from noise. Furthermore, there is a growing literature developing theoretical foundations of score-based diffusion models \cite{chen2023sampling,benton2023linear,lee2023convergence,chen2023improved,chen2024learning}, giving a comprehensive error analysis including score estimation, initialization error and time-discretization error. By generating high-fidelity data, these models can also serve as data prior for posterior sampling in inverse problems in high dimensions. Following this idea, many studies (see, e.g., \cite{kawar2021snips,chung2023diffusion}) have leveraged diffusion models for posterior sampling. However, as discussed below, various categories of approaches for posterior sampling introduce different uncontrollable errors, such as those arising from the approximation of the conditional score or the use of a limited variational family. This abundance of heuristics contrasts with the principled sampling used in prior data generation, and is often at odds with the statistical guarantees needed in many scientific applications.

In this work, we aim to bridge the gap between principled diffusion-based algorithms for both prior and posterior distributions. Focusing on the canonical setting of linear inverse problems, where measurements are of the form $y = Ax + w$, with $x \sim \pi$ the signal to be estimated and $w$ an independent noise, we first illustrate a negative result, revealing that no method can efficiently sample from the posterior distribution in general cases, even with the prior denoising oracle. We next develop the \textit{tilted transport} technique, which utilizes the quadratic structure of the log-likelihood in linear inverse problems in combination with the prior denoising oracle to exactly transform the original posterior sampling problem into a new one that is easier to sample. \Cref{fig:schemeplot} illustrates a schematic plot of the method using two-dimensional Gaussian mixture examples, showing that while the original target posterior problem remains multimodal, the boosted posterior resembles a unimodal distribution.

\begin{figure}[!ht]
\centering
\includegraphics[width=0.99\textwidth]{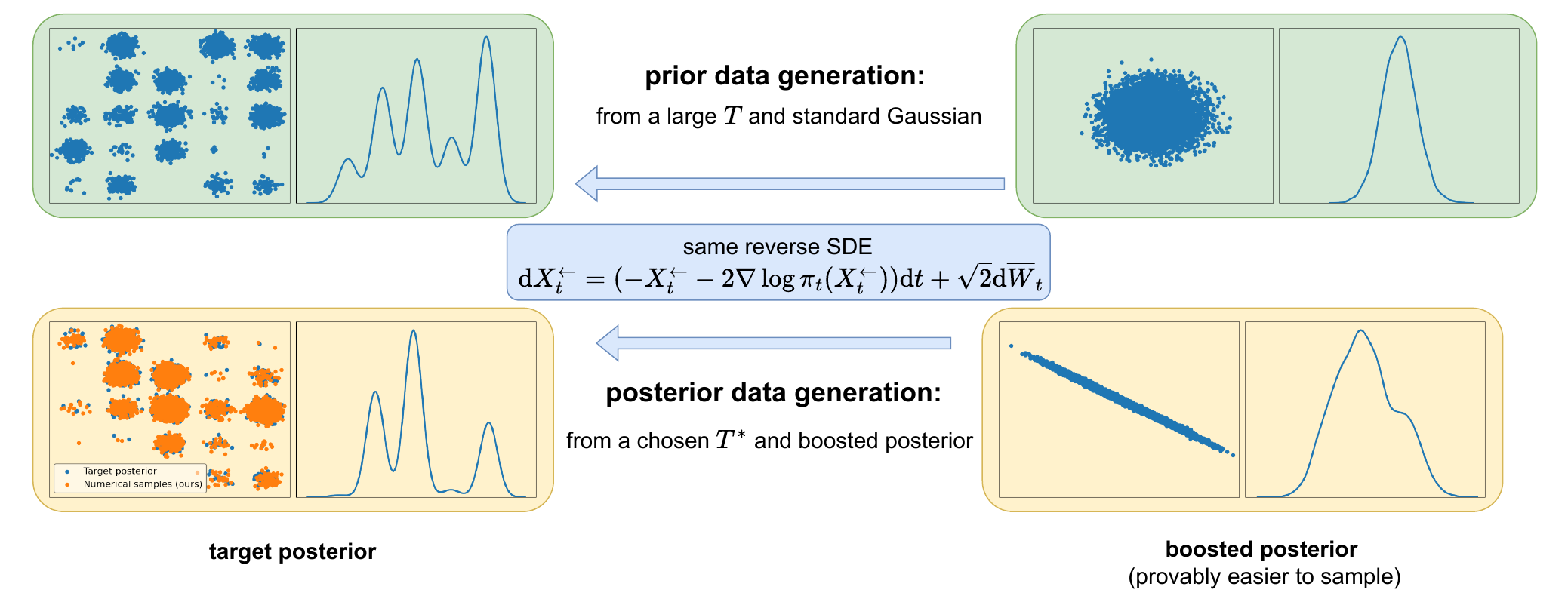}
\caption{Schematic plot of tilted transport boosting posterior sampling with a 2D Gaussian mixture example. The density plot shows the first variable's density, and the scatter plot displays the samples.
}
\label{fig:schemeplot}
\end{figure}

We establish a sufficient condition where the density of the transformed posterior problem becomes strongly log-concave, making it suitable for efficient sampling via Langevin dynamics thanks to the Bakry-Emery criterion \cite{bakry2006diffusions}.  This condition showcases the interplay between a geometric property of the prior (the so-called \emph{susceptibility}  $\chi_t(\pi)$; see \Cref{sec:quant}) and the conditioning and noise level of the measurements. 
Interestingly, the condition can be satisfied when the signal-to-noise ratio (SNR) is either moderately low \emph{or} moderately high, in contrast with traditional sampling methods, which typically excel only within a specific regime. 

As a first application, we show that tilted transport can sample 
from Ising models of the form $\nu(x) \propto e^{-\frac12 x^\top Q x}$, where $x \in \{\pm 1\}^d$ is supported in the hypercube, up to the critical threshold determined by the gap $\lambda_{\max}(Q) - \lambda_{\min}(Q) =1$, thus matching the performance of Glauber dynamics \cite{el2022sampling, anari2021entropic} as well as the computational threshold predicted by the low-degree method \cite{doi:10.1137/1.9781611977912.180}. The sampling guarantees are directly based on the Bakry-Emery criterion applied to the tilted posterior, and avoid the need to perform multiscale extensions of this criterion based on localization schemes \cite{chen2022localization} or the Polchinksy renormalization group \cite{bauerschmidt2023stochastic}.

More generally, even when the boosted posterior is not strongly log-concave, it is more amenable to sample than the original one. %
Thus, tilted transport can be combined with any existing black-box posterior sampling methods to enhance their performance. This technique operates without any additional computational cost and functions in a plug-and-play fashion, allowing for straightforward integration into various frameworks. When working with high-dimensional Gaussian mixtures, where an analytical solution to the posterior is available, we numerically validate our theory and demonstrate enhanced posterior sampling performance. In the two-dimensional lattice scalar $\varphi^4$ field model, we also observe that the Langevin dynamics exhibit accelerated relaxation times with the boosted distribution compared to the original distribution. %

\subsection{Related Work} 
Numerous studies in recent years have explored score-based priors for posterior sampling.  We note that several recent works~\cite{song2022solving,choi2021ilvr,graikos2022diffusion,shoushtari2022dolph,chung2022come} introduce hyperparameters to balance the influence of the prior and measurements, resulting in sampling strategies that guide output to regions where the given observation is
more likely. These strategies typically deviate from the principles of Bayesian posterior sampling and often lack a precise definition of the resulting distribution. In contrast, other approaches adhere more closely to Bayesian principles. One such approach is variational inference, which involves designing variational objectives and optimization methods based on the structure of score-based diffusion~\cite{kawar2022denoising,mardani2023variational,feng2023score,janati2024divide}. However, even with an accurate prior score, the accuracy of posterior sampling heavily depends on the choice of variational family and optimization procedures, not to mention the additional optimization cost. Another popular strategy focuses on approximating the score conditional on the measurement using various heuristics~\cite{song2020score,kawar2021snips,jalal2021robust,chung2023diffusion,meng2022diffusion,song2022pseudoinverse,song2023loss}. In this approach, approximation errors typically remain largely uncontrollable due to the challenges associated with tracking the conditional distribution for intermediate states. Recently, some studies have adopted sequential Monte-Carlo methods to systematically approximate the conditional score~\cite{wu2024practical,cardoso2024monte,dou2024diffusion}, providing consistency as the number of particles used to approximate the conditional distribution of the intermediate states increases. However, this particle-based method still struggles with high-dimensional problems due to the curse of dimensionality~\cite{bickel2008sharp}. We note that \cite{cardoso2024monte,heurtel2024listening} also intuitively explores the possibility of reducing the original posterior to an equivalent one under restrictive conditions in the discrete time setting or purely denoising setting. In contrast, our tilted transport technique operates in a fairly generic setting and is supported by a clear theoretical foundation.

\paragraph{Notations:} $\mathcal{P}(\R^d)$ denotes the space of probability measures over $\R^d$, and $\mathcal{P}_2(\R^d)$ denotes those measures with finite second-order moments. $\gamma_d$ denotes the $d$-dimensional standard Gaussian measure, and by slight abuse of notation, $\gamma_{\delta}$ or $\gamma_{\Sigma}$ denote the centered Gaussian measure with covariance $\delta I_d$ or $\Sigma$ when the context is clear. For $Q \succeq 0$ in $\R^{d \times d}$ and $b \in \R^d$ in the span of $Q$, the \emph{quadratic tilt} of $\pi$ is the measure $\mathsf{T}_{Q,b} \pi \ll \pi$ with density proportional to $\frac{d\mathsf{T}_{Q,b} \pi}{d\pi}(x) \propto \exp\left\{ -\frac12 x^\top Q x + x^\top b\right\}$. We also use the notation $\mathsf{T}_{Q}$ when $b=0$. $\|Q\|$ denotes its operator norm.
$\pi \ast \gamma$ denotes the convolution of two measures $\pi$ and $\gamma$.
For $\alpha \geq 0$ and $\pi \in \P(\R^d)$, we define $\D_\alpha \pi(x) := \alpha^d \pi(\alpha x)$ as the {dilation} of $\pi$. 
For $\beta \geq 0$ and $\pi \in \P(\R^d)$, we define $\C_\beta \pi(x) := \pi \ast (\D_{\beta^{-1/2}} \gamma_d)$ as the {Gaussian convolution} of $\pi$.

\paragraph{Acknowledgements:} We thank Stéphane Mallat and Maarten de Hoop for fruitful discussions during the earlier phase of this project. We thank Ahmed El Alaoui for pointing out the relevant work \cite{doi:10.1137/1.9781611977912.180} and for engaging and stimulating discussions. We thank Eric Vanden-Eijnden, Jonathan Niles-Weed, Andrea Montanari, Nick Boffi, Jianfeng Lu, Florentin Guth, Mark Goldstein, Brian Trippe, Benoit Dagallier and Roland Bauerschmidt for useful feedback during the completion of this work. JB was partially supported by the Alfred P. Sloan Foundation
and awards NSF RI-1816753, NSF CAREER CIF 1845360, NSF CHS-1901091 and NSF DMS-MoDL
2134216.

\section{Preliminaries}
\label{sec:preliminary}

\paragraph{Problem Setup}
Consider a high-dimensional object of interest $x \in \R^d$, drawn from a certain probability distribution $\pi \in \mathcal{P}(\mathbb{R}^d)$. We suppose that one has 
managed to learn a generative model for $\pi$ via the DDPM objective~\cite{ho2020denoising}; in other words, 
for any $y \in \R^d$ and $\sigma \geq 0$, we have access to the \emph{denoising oracle} $\mathsf{DO}_\pi(y,\sigma) := \mathbb{E}[x | y]$, where $y = x + \sigma w$, with $x \sim \pi$ and $w \sim \gamma_d$ independent. 
It is by now well-established that such denoising oracle enables efficient sampling of $\pi$, well beyond the classic isoperimetric assumptions for fast relaxation of Langevin dynamics \cite{chen2023sampling}.

Suppose that we now measure $y = A x + \sigma w$, where again $x\sim \pi$ and $w \sim \gamma_d$ are independent, but now $A \in \mathbb{R}^{d' \times d}$ is a \emph{known} linear operator different from the identity. Given these linear measurements, we are now interested in the \emph{posterior sampling} of $x$ given $y$. This corresponds to the basic setup of linear inverse problems, encompassing many applications such as image inpainting, super-resolution, tomography, or source separation, to name a few. We are interested in the following natural question: can the power of denoising oracles be provably transferred to posterior sampling?

By Bayes' rule, the posterior distribution $\nu_{y, A}$ (denoted simply by $\nu$ when the context is clear) has density proportional to $ \pi(x) p( y | x) \propto \exp\left\{ -\frac{1}{2\sigma^2} \| Ax - y \|^2\right\} \pi(x)$, and thus we can write it as a quadratic tilt of $\pi$:
$$\nu = \mathsf{T}_{Q, b} \pi~,\text{ with } Q = \sigma^{-2} A^\top A\,,\, b = -\sigma^{-2}A^\top y~.$$
We readily identify certain regimes where sampling from $\nu$ might be easy: 
\begin{itemize}
    \item If $\lambda_{\min}(Q)$ is sufficiently large, $\lambda_{\min}(Q) \gg 1$, then one expects $\nu$ to be strongly log-concave, enabling fast relaxation of Langevin dynamics.
    \item If $\lambda_{\max}(Q)$ is sufficiently small, $\lambda_{\max}(Q) \ll 1$, then one expects $\nu \approx \pi$ in the appropriate sense, and therefore that samples from $\pi$ (which can be produced efficiently thanks to $\mathsf{DO}_\pi$) may be perturbed into samples from $\nu$.
    \item If $A \in \mathcal{O}_d$ is a unitary transformation, then $Q=\mathrm{Id}$ and the inverse problem reduces to isotropic Gaussian denoising, and is thus at first glance `compatible' with the structure of the denoising oracle (such observation will be formalized later). 
\end{itemize}
At this stage, we can already identify two key parameters of the problem that are likely to drive the difficulty of posterior sampling: on one hand, a proxy for the signal-to-noise ratio, measured e.g., by $\mathrm{SNR}:= \lambda_{\min}(Q) = \frac{\lambda_{\min}(A)^2}{\sigma^2}$. On the other hand, the conditioning of the measurement operator $A$, $\kappa(A) := \frac{\lambda_{\max}(A)}{\lambda_{\min}(A)}$. As we shall see, these two characteristics of the linear measurement system will characterize necessary and sufficient conditions for probable posterior sampling. In the following, we assume the log of prior density $\pi$ is smooth and its Hessian exists $\forall x \in \R^d$.

\paragraph{Denoising Oracles and Score-Based Diffusion}
Let us first review the natural connection between denoising and score-based generative modeling. Score-based diffusion models consist of two processes: a forward process that gradually adds noise to input data and a reverse process that
learns to generate data by iteratively removing this noise. For example, one widely used family for the forward process is the Ornstein–Uhlenbeck (OU) process\footnote{
In practice, it is also common to introduce a positive smooth function
$\beta$: $\R_+ \rightarrow \R_+$ and consider the time-rescaled OU process $\dd X_t = -\beta(t)X_t \dd t + \sqrt{2\beta(t)}\dd W_t$.
Our results could be applied directly to these variants by rescaling time. For the ease of notation, we keep $\beta(t) \equiv 1$ in the main text.
}:
\begin{align}
    \dd X_t = -X_t \dd t + \sqrt{2}\dd W_t,  \quad X_0 \sim \pi~,
\end{align}
where $W_t$ is the standard Wiener process. We use $\pi_t$ to denote the density of $X_t$, 
given by the action of the OU semigroup $\pi_t = \mathsf{O}_t^* \pi$, defined by $\mathsf{O}_t f(x)= \mathbb{E}[ f(X_t) | X_0 = x]$, and explicitly given by dilated Gaussian convolutions, 
 $\mathsf{O}_t^* := \C_{\beta_t} \D_{\alpha_t}$, with $\beta_t = 1 - e^{-2t}$ and $\alpha_t = e^t$. 
With a sufficiently large $T$, we know that $\pi_T$ is close to the density of standard Gaussian $\gamma_d$, owing to the exponential contraction of the OU semigroup: $\mathrm{KL}( \pi_T || \gamma_d) \leq e^{-T} \mathrm{KL}(\pi || \gamma_d)$.

Finally, the measure $\pi_t$ solves the Fokker-Plank equation 
\begin{equation}
\label{eq:OUFPE}
    \partial_t \pi_t = \nabla \cdot ( x \pi_t) + \Delta \pi_t~,~\pi_0 = \pi~.
\end{equation}

By writing (\ref{eq:OUFPE}) as a transport equation $\partial_t \pi_t = \nabla \cdot ( (x + \nabla \log \pi_t) \pi_t) $, we can formally reverse the transport starting at a large time $T$ and solving 
\begin{equation}
\label{eq:transport_backwards}
    \partial_t \tilde{\pi}_t = \nabla \cdot ( -(x + \nabla \log \pi_{T-t}) \tilde{\pi}_t) ~,~\tilde{\pi}_0 = \pi_T~. 
\end{equation}
Since $\tilde{\pi}_t = \pi_{T-t}$ for $ 0 \leq t \leq T$, introducing again the dissipative term leads to 
$
    \partial_t \tilde{\pi}_t = \nabla \cdot ( -(x + 2 \nabla \log \tilde{\pi}_{t}) \tilde{\pi}_t) + \Delta \tilde{\pi}_t~,~\tilde{\pi}_0 = \pi_T~,
$
which admits the SDE representation 
\begin{equation}
    \dd \tilde{X}_t = (\tilde{X}_t + 2\nabla \log \pi_{T-t}(\tilde{X}_t))\dd t + \sqrt{2}\dd \overline{W}_t,~~\tilde{X}_0 \sim \pi_T~.
\end{equation}
In practice, one runs this reverse diffusion starting from $\tilde{X}_0 \sim \gamma_d$ rather than $\tilde{X}_0 \sim \pi_T$. However, by the data-processing inequality, we have that $\mathrm{KL}( \pi || \tilde{\pi}_T) \leq \mathrm{KL}( \pi_T || \gamma_d) = O( e^{-T})$, thus incurring in insignificant error. To facilitate later exposition, we write the above process reverse in time~\citep{anderson1982reverse,haussmann1986time}
\begin{equation}
    \dd X_t\reverse = (-X\reverse_t - 2\nabla \log \pi_t(X\reverse_t))\dd t + \sqrt{2}\dd \overline{W}_t,\label{eq:reverse_sde}~~X_T\reverse \sim \gamma_d,
\end{equation}
and interpret the data generation process as running the reverse SDE from $T$ back to 0. 

By the well-known Tweedie's formula, and up to time reparametrisation, the denoising oracle is equivalent to the time-dependent score $\nabla \log \pi_t$:
\begin{fact}[Tweedie's formula, {\cite{herbert1956empirical}}]
We have $\nabla \log \pi_t(x) = -(1-e^{-2t})^{-1}( x - e^{-2t}\mathsf{DO}_\pi(x, 1-e^{2t}) )$.
\end{fact}

    The OU semigroup corresponds to the so-called `variance-preserving' diffusion scheme. 
    Instead one could also consider the `variance-exploding' scheme, given directly by the Heat semigroup $\mathsf{H}_t^* \pi = \pi \ast \gamma_{2t}$. One can immediately verify that denoising oracles are equally valid to reverse the associated Fokker-Plank transport equation, given by $\partial_t \pi_t = \Delta \pi_t$. Indeed, the reverse SDE now reads
\begin{equation}
    \dd X_t\reverse = (- 2\nabla \log \pi_t(X\reverse_t))\dd t + \sqrt{2}\dd \overline{W}_t,\label{eq:reverse_sde_heat}~~X_T\reverse \sim \gamma_d~.
\end{equation}

\paragraph{Log-Sobolev Inequality and Fast Relaxation of Langevin Dynamics}
Given a Gibbs distribution $\pi \in \mathcal{P}(\R^d)$ of the form $\pi \propto e^{-f}$, a powerful and versatile method to sample from $\pi$ is to consider the Langevin dynamics 
\begin{align}
\label{eq:Langevin}
    \dd X_t &= -\nabla f(X_t) \dd t + \sqrt{2} \dd W_t~,~X_0 \sim \mu_0~,
\end{align}
where $\mu_0$ is an arbitrary initial distribution. It is easy to verify that these dynamics define a Markov process that admits $\pi$ as its unique invariant measure. Perhaps less obvious is the fact that the Fokker-Plank equation associated with \cref{eq:Langevin}, given by $\partial_t \mu = \nabla \cdot ( \nabla f \mu) + \Delta \mu $ (and where $\mu_t$ is the law of $X_t$) is in fact a Wasserstein gradient flow for the relative entropy functional $\mathrm{KL}( \mu || \pi)$ \cite{jordan1998variational}. Under this interpretation, one can quantify the convergence of Langevin dynamics to their invariant measure, i.e., its time to relaxation, by establishing a sharpness or \emph{Polyak-Lowacjevitz} (PL)-type inequality. Indeed, by noticing that $\frac{d}{dt} \mathrm{KL}( \mu || \pi) = - \mathrm{I}( \mu || \pi)$, where $ \mathrm{I}( \mu || \pi)= \mathbb{E}_{\mu} [\| \nabla \log \mu - \nabla \log \pi \|^2] $ is the Fisher divergence, the PL-type inequality in this setting is given by the  \emph{Logarithmic Sobolev Inequality} (LSI)~\cite{gross1975logarithmic,ledoux2001concentration}: we say that a measure $\pi$ satisfies $\mathrm{LSI}(\rho)$ if for any $\mu \in \mathcal{P}(\R^d)$ it holds 
\begin{equation}
    \mathrm{KL}( \mu || \pi) \leq \frac{1}{2 \rho} \mathrm{I}( \mu || \pi)~.    
\end{equation}
This functional inequality \footnote{Note that we use the convention from \cite{markowich2000trend} where the relevant direction is to lower bound $\rho$, as opposed to the other common direction $\mathrm{KL}( \mu || \pi) \leq \frac{\tilde{\rho}}{2} \mathrm{I}( \mu || \pi)$, e.g. \cite[Definition 5.1.7]{bakry2014analysis}. } directly implies $\mathrm{KL}( \mu_t || \pi) \leq e^{-2\rho t } \mathrm{KL}( \mu_0 || \pi)$. 
While for general $\pi$ it is typically hard to establish, there are two important sources of structure that lead to quantitative (i.e., $\rho=\Omega_d(1)$) bounds: when $\pi$ is a product measure $\pi=\tilde{\pi}^{\otimes d}$ (in which case $\pi$ satisfies LSI with the same constant as $\tilde{\pi}$), and when $\pi$ is strongly log-concave\footnote{or a suitable perturbation of it via the Hooley-Strook perturbation principle \cite{holley1986logarithmic} }, i.e., $-\nabla^2 \log \pi(x) \succeq \alpha I$ for all $x$, in which case the celebrated Bakry-Emery criterion~\cite{bakry2006diffusions} states that $\rho \geq \alpha$.

\section{Evidence of Computational Hardness in the Generic Case}
\label{sec:lowerbounds}

We start our analysis of posterior sampling by discussing negative results 
for the general case. Recently, \cite{gupta2024diffusion} established computational 
lower bounds for this task using cryptographic hardness assumptions. In this section, we complement these results by illustrating a correspondence with sampling 
problems on Ising models, leading to an arguably simpler conclusion. 

For this purpose, consider $\bar{\pi} = \mathrm{Unif}(\{ \pm 1\}^d)$ 
 the uniform measure of the hypercube. Quadratic tilts of $\bar{\pi}$ define generic Ising models, a rich and intricate class of high-dimensional distributions.
Since $\bar{\pi}$ is a product measure, its associated denoising oracle becomes a separable function that can be computed in closed-form:
\begin{fact}[Denoising Oracle for $\bar{\pi}$]
Let $\gamma(t; \mu, \sigma)=\exp\left\{-\frac{1}{2\sigma^2}(t-\mu)^2 \right\}$. Then we have 
\begin{align}
    \mathsf{DO}_{\bar{\pi}}(y, \sigma) &= \left( \phi( y_i; \sigma) \right)_{i=1\ldots d}~, \text{ with }~    \phi(t, \sigma) = \frac{\gamma(t, +1, \sigma) - \gamma(t, -1, \sigma)}{\gamma(t, +1, \sigma) + \gamma(t, -1, \sigma)}~.
\end{align}
\end{fact}

Given a symmetric matrix $Q \in \mathbb{R}^{d \times d}$, an Ising model is given by the tilt $\mathsf{T}_Q \bar{\pi} \in \mathcal{P}( \{\pm 1\}^d)$. In our setting, we can thus view such models as the posterior distribution 
of a linear inverse problem associated with the uniform prior $\bar{\pi}$. 
Efficiently sampling from Ising models is a fundamental question at the interface of statistical physics and high-dimensional probability, and several works provide evidence of computational hardness under a variety of settings. 

Notably, by treating $Q$ as the adjacency matrix of a regular graph, \cite{galanis2016inapproximability} establishes that sampling from $\nu$ is impossible for $\lambda_{\max}(Q) - \lambda_{\min}(Q) \geq 2 + \varepsilon$, for any $\varepsilon>0$, unless $\mathsf{NP} = \mathsf{RP}$. In other words, for poorly conditioned tilt $Q$ (in the sense that there is a large gap between the smallest and largest eigenvalue), there is no efficient posterior sampling algorithm, \emph{even with the knowledge of the prior denoising oracle}. 
The threshold can even be reduced to $1 + \varepsilon$ by using a weaker notion of computational hardness \cite{doi:10.1137/1.9781611977912.180}, given by the \emph{low-degree polynomial method} \cite{barak2019nearly, kunisky2019notes}. 
Remarkably, this threshold agrees with the current best-known algorithmic results for sampling generic Ising models with Glauber dynamics \cite{eldan2022spectral,anari2021entropic,bauerschmidt2019very}. 
Finally, we also mention that when $Q$ is a random Gaussian symmetric matrix, the associated model is the so-called Sherrington-Kirkpatrick (SK), which has been analyzed in \cite{el2022sampling, celentano2024sudakov}. In this setting, these works establish that `stable' sampling algorithms
 \footnote{Defined as sampling algorithms whose output law depends smoothly on $Q$ in the Wasserstein metric} 
fail to sample from the SK model as soon as $\lambda_{\max}(Q) - \lambda_{\min}(Q) > 4$, and that this threshold can be reached 
with dedicated sampling algorithms based on AMP.

In summary, we have: 
\begin{fact}[Computational Hardness of Sampling Ising Models, \cite{doi:10.1137/1.9781611977912.180, galanis2016inapproximability}]
There exist no general-purpose, efficient posterior sampling algorithms, for $Q$ sufficiently ill-conditioned, even under the knowledge of the prior denoising oracle. 
\end{fact}

One could wonder whether this computational hardness comes from the discrete nature of the hypercube. 
It is not hard to observe that this is not the case: the following proposition, proved in Appendix \ref{app:proof_smooth}, shows a simple reduction from a model where the prior $\bar{\pi}$ is replaced by a smooth mixture of Gaussians $\pi$ centered at the corners of the hypercube, with variance $\delta$.  
\begin{proposition}[Hardness extends to smooth priors]
\label{prop:hardness_smooth}
  Assume a posterior sampler exists for the smooth prior with TV error $\epsilon$ and $\delta = o(d^{-1/2})$. Then there exists a sampler for the associated Ising model with TV error $1.1 \epsilon$. 
\end{proposition}

In conclusion, one cannot hope for a generic method that leverages the prior denoising oracle to perform efficient posterior sampling, as soon as $A$ is mildly ill-conditioned. 
Thus, in order to perform provable posterior sampling, one needs to either (i) constraint the measurements, or (ii) exploit structural properties of the prior measure. 
In the following, we focus on (i), namely providing guarantees for well-conditioned $A$ that leverage the OU semigroup for generic prior distributions.

\section{Posterior Sampling via Tilted Transport}
\label{sec:tilttrans}

We now present a simple method that reduces the original posterior sampling problem 
to another posterior sampling problem with more benign geometry, by leveraging the 
shared quadratic structure of the posterior tilt and the OU semigroup. 
The power of the denoising oracle to perform sampling of the prior $\pi$ comes from its 
ability to run the transport equation (\ref{eq:transport_backwards}) in either direction, 
and leveraging the fact that sampling from $\pi_T$ is easy. To transfer this power 
to posterior sampling, we can thus attempt to replicate this scheme: can we implement a 
transport between the posterior $\nu$ and a terminal measure $\nu_T$ that is easy to sample, that only relies on the pre-trained prior $\mathsf{DO}_\pi$? 

\paragraph{A Motivating Example} Consider first the denoising setting: $y=x+\sigma w$. According to the forward OU process, we have $p(X_{s}|X_0) \eqind \mathcal{N}(e^{-s}X_0, (1-e^{-2s})I_d)$. Introduce $T>0$ and define $\tilde{y}=e^{-T}y = e^{-T}x + e^{-T}\sigma w$ such that $p(\tilde{y}|x) \eqind \mathcal{N}(e^{-T}x, e^{-2T}\sigma^2 I_d)$. We match the variance by letting $e^{-2T}\sigma^2=1-e^{-2T}$, i.e., $T=\frac12 \log(1+\sigma^2)$, so $p(\tilde{y}|x)=p(X_{s}|X_0)$, and thus $(x, \tilde{y}) \eqind (X_0, X_{T})$. Therefore, to perform the posterior sampling $p(x|\tilde{y})$, we only need to do the sampling $p(X_0|X_{T})$, which can be achieved through the reverse SDE. Specifically, let $X_{T} = e^{-T}y$ and run the reverse SDE \eqref{eq:reverse_sde} from $T$ to 0, then $X_0$ will be the desired posterior.

\paragraph{Hamilton-Jacobi Equation and Quadratic Tilts}
If $\pi_t$ solves the Fokker-Plank~\cref{eq:OUFPE}, then one can verify that the time-varying potentials $f_t := \log \pi_t$ solve the viscous Hamilton-Jacobi PDE (HJE)
\begin{equation}
\label{eq:HJ}
    \partial_t f_t = \Delta f_t + \| \nabla f_t \|^2 + x \cdot \nabla f_t~,~f_0 = f~.
\end{equation}
In the heat semigroup setting, one obtains a closely related HJE without the last linear term. Now, the posterior $\nu = \mathsf{T}_{Q,b} \pi$ creates an additional quadratic term in the potential $\log \nu = f -\frac12 x^\top Q x + x \cdot b$. 
One could naively hope that this additive quadratic term would still define a solution of the HJE with the tilted initial condition $\tilde{f}_0 = \log \nu$ --- or equivalently that the measure $ \mathsf{T}_{Q,b} \pi_t$ solves the transport equation (\ref{eq:transport_backwards}). 
Unfortunately, due to the nonlinearity in (\ref{eq:HJ}) brought by the terms $\|\nabla f_t\|^2$, this is not the case. However, as we shall see now, this is not far from being true: one just needs to consider \emph{time-varying} quadratic tilts in order to satisfy the HJE.

\paragraph{Tilt Transport Equation}
We consider then a one-parameter family of distributions $\nu_t$ of the form 
\begin{equation}
\nu_t := \mathsf{T}_{Q_t, b_t} \pi_t~,~ \text{ with }~Q_0 = Q~,~ b_0 = b~.    
\end{equation}
As it turns out, one can ensure that $\log \nu_t$ solves the HJE associated with the reverse OU process by asking that $Q_t, b_t$ satisfy the first-order ODE:
\begin{equation}
\begin{cases}
    \dot{Q}_t = 2(I+Q_t)Q_t ~, &Q_0 = Q\\
    \dot{b}_t = (I+2Q_t)b_t  ~,&b_0 = b
\end{cases}
\label{eq:boost_ode}
\end{equation}

\begin{theorem}[Tilted Transport under OU Semigroup]
\label{thm:density_equivalence}
Assume $t<T$ such that the ODE \eqref{eq:boost_ode} is well-defined on $[0, t]$. By initializing $X_t \sim \nu_t$ and run the reverse SDE~\eqref{eq:reverse_sde} from $t$ to 0, we have $X_s \sim \nu_s$ for $s\in[0, t]$, specifically, $X_0$ gives the desired posterior.
\end{theorem}

\paragraph{Solution to \cref{eq:boost_ode}} Without loss of generality, we assume $d'\leq d$, and the observation operator $A \in \R^{d'\times d}$ has a general singular value decomposition form $A=U\Sigma V^\top$ with non-zero singular values $\lambda_1 \geq \lambda_2 \geq \dots \geq \lambda_{d'} > 0$.  
By diagonalizing $Q$ and solving the scalar ODE $\dot{q}_t = 2(1+q_t)q_t$ for diagonal entries, we have
$
    Q_t = V \text{diag}\left( \frac{e^{2t}}{1 + \sigma^2/\lambda_1^2-e^{2t}}, \cdots, \frac{e^{2t}}{1 + \sigma^2/\lambda_{d'}^2-e^{2t}}, 0, \cdots, 0 \right) V^\top,
$
where the solution is defined up to the blowup time $T:=\frac12 \log(1+\sigma^2/\lambda_1^2) = \frac12 \log(1+\lambda_{\max}(Q)^{-1})$.
$b_t$ can be further solved from the solution $Q_t$; see \Cref{app:ode_solution} for more details.

With the explicit solution of $Q_t, b_t$, we can interpret the term $\exp(-\frac12 x^\top Q_t x + x^\top b_t)$ as the likelihood of the inverse problem with respect to the new prior distribution $\pi_t$ and the corresponding operator. Based on this observation and \Cref{thm:density_equivalence}, we have the following corollary, transforming the original posterior sampling problem to a new posterior sampling problem exactly. We remark that when $A$ is identity, the corollary recovers the analysis we have in the motivating example; see \Cref{app:ode_solution} for the proof and more discussions.

\begin{corollary}[Posterior Sampling via Tilted Transport]
\label{coro:boosted_posterior}
Fix $t \leq T$. Sampling from the original posterior $\nu = \mathsf{T}_Q \pi$ is equivalent to a two-step process: first, sample from a new posterior $X_t \sim \nu_t$, and then run the reverse SDE~\eqref{eq:reverse_sde} from time $t$ to 0. 
\end{corollary}

\begin{remark}[Tilted Transport under Heat Semigroup]
If we consider the `variance-exploding' setting, in which the prior evolves along the Heat semigroup $\pi_t = \pi \ast \gamma_{2t}$, the tilted transport equation has the same quadratic structure: 
\begin{equation}
\begin{cases}
    \dot{Q}_t = 2Q_t^2 ~, &Q_0 = Q\\
    \dot{b}_t = 2Q_tb_t  ~,&b_0 = b
\end{cases}
\label{eq:boost_ode_heat}
\end{equation}
which blows up in time $T = \|Q\|^{-1}/2$, and has the analytic form $Q_t = V \text{diag}( \lambda_i (1-2 t \lambda_i)^{-1} ) V^\top$. See \Cref{app:heat_proof} for more details.   
\end{remark}

\begin{remark}[Covariance Decomposition of \cite{bauerschmidt2019very} and Polchinsky Flow]
\label{rem:polchinsky}
    In \cite{bauerschmidt2019very}, the authors develop a transformation of the tilt by $Q$ via a decomposition of its associated covariance $Q^{-1}$. Specifically, they consider a decomposition of the form $Q^{-1} = \|Q\|^{-1} \mathrm{Id} + B^{-1}$, which expresses the Gaussian measure with covariance $Q^{-1}$ as the convolution of two Gaussian measures, with covariance $\|Q\|^{-1} \mathrm{Id}$ and $B^{-1}$ respectively. 
    Our modified tilt at blowup $Q^\star := Q_{T}$ is precisely $Q^\star = B$ in the Heat semigroup setting (and $Q^\star=B( 1 + \|Q\|^{-1})$ in the OU setting). 
    
    In that case, in the language of the Polchinsky flow of \cite{bauerschmidt2023stochastic}, given the original measure $\nu = \mathsf{T}_Q \pi$, the measure at blowup is $\nu_* = \mathsf{T}_{Q^\star}( \pi \ast \gamma_{\|Q\|^{-1}})$, which 
    can be viewed as the \emph{renormalised measure}  
    $ \mathsf{T}_{(C_\infty - C_{\tilde{t}})^{-1} } ( \pi \ast \gamma_{C_{\tilde{t}}})$ corresponding to $C_\infty = Q^{-1}$ and $C_{\tilde{t}} = \tilde{t} \mathrm{Id}$. In other words, we run the isotropic Polchinsky flow $\dot{C}_t = \mathrm{Id}$ for $t \leq \tilde{t}$ as long as $C_\infty - C_t \succeq 0$, which happens precisely at $\tilde{t}=\|Q\|^{-1}$.
\end{remark}
Collecting these remarks thus leads to the following.
\begin{corollary}[Posterior Sampling via Tilted Transport, Heat Semigroup Setting]
\label{coro:boosted_heat}
Sampling from the posterior $\nu=\mathsf{T}_Q \pi$ can be achieved  
by first sampling from $\nu_* := \mathsf{T}_{Q_*}( \pi \ast \gamma_{\|Q\|^{-1}})$, 
where $Q_*^{-1} = Q^{-1} - \|Q\|^{-1} \mathrm{Id}$, and then running the reverse SDE \eqref{eq:reverse_sde_heat} from time $\|Q\|^{-1}/2$ to $0$. 
\end{corollary}

\section{Quantitative Conditions for Provable Sampling}
\label{sec:quant}
The new posterior sampling problem described above may be easier to sample than the original posterior sampling problem due to two separate aspects. On the one hand, the (negative) eigenvalues of the quadratic tilt $-\frac12 x^\top Q_t x + x^\top b_t$ become more negative, essentially meaning that the SNR of the new observation model becomes larger. To be more specific, as $t\rightarrow T$, $\lambda_{\min}(Q_t) \rightarrow \frac{1+\lambda_{\max}(Q)^{-1}}{\lambda_{\min}(Q)^{-1}-\lambda_{\max}(Q)^{-1}} > \lambda_{\min}(Q)$. On the other hand, the new prior distribution $\pi_{T}$ becomes closer to a single-mode Gaussian (recall that $\mathrm{KL}(\pi_t || \gamma_d) =O( e^{-t}) $), which is also easier to sample.
Let us now quantify the above intuition by leveraging the Bakry-Emery criterion. 

\subsection{Sufficient Conditions via Barky-Emery}

We start by giving a simple sufficient condition that ensures that $\nu_{T}$ is strongly log-concave. As discussed earlier, by the Bakry-Emery criterion, this ensures fast relaxation of the Langevin dynamics, enabling efficient sampling from $\nu_{T}$ -- and therefore of $\nu$ as per Corollary \ref{coro:boosted_posterior}. 
For that purpose, given the prior $\pi \in \mathcal{P}(\R^d)$ and $t \geq 0$, we define 
\begin{equation}
\label{eq:chi}
\chi_t(\pi):= \sup_{x \in \R^d} \| \mathrm{Cov}[ \mathsf{T}_{tI, tx} \pi] \|,
\end{equation}
where the covariance is given by $\mathrm{Cov}[\mu] = \mathbb{E}_{x \sim \mu}[xx^\top] - (\mathbb{E}_{x \sim \mu}[x])(\mathbb{E}_{x \sim \mu}[x])^\top$.  
$\chi_t(\pi)$ thus measures the largest `spread' of any tilted measure of the form $\mathsf{T}_{t,x}\pi$, and is also known as the \emph{susceptibility} in certain field models \cite{bauerschmidt2022log}. 
The covariance of isotropic tilts is a central object in the \emph{stochastic localization} framework of Eldan \cite{eldan2020taming,chen2022localization}, as well as the Polchinsky renormalisation group approach of \cite{bauerschmidt2019very,bauerschmidt2023stochastic}.
Equipped with this definition, we have the following simple sufficient condition to ensure that $\nu_{T}$ is strongly log-concave: 
\begin{proposition}[Strong Log-Concavity of $\nu_{T}$]
\label{thm:strong_logconcave_chi}
    Let $\kappa = \lambda_{\max}(Q) / \lambda_{\min}(Q)$ denote the condition number of $Q$. Then $\nu_{T}$ is strongly log-concave if
    \begin{equation}
    \label{eq:necessary_cond_general}
    \chi_{\|Q\|}(\pi) < \|Q\|^{-1} \frac{\kappa }{\kappa - 1}~.
    \end{equation}
    When $Q$ is degenerate $(\lambda_{\min}(Q)=0$, or equivalently $\kappa=\infty$), $\nu_{T}$ is strongly log-concave if
    \begin{equation}
        \chi_{\|Q\|}(\pi) < \|Q\|^{-1}.
    \end{equation}
\end{proposition}
The proof is in \Cref{app:strong_logconcave_proof_chi}. By the Barky-Emery criterion,  \Cref{eq:necessary_cond_general} is sufficient to ensure a fast relaxation of the Langevin dynamics with drift $\nabla \log \nu_T$, which can then be used to produce samples of $\nu$ thanks to \Cref{coro:boosted_posterior}. One can also verify that the same exact condition implies that $\nu_*$ (defined in \Cref{coro:boosted_heat}) is strongly log-concave, thus the conclusion also applies to the Heat semigroup setting. 

\Cref{thm:strong_logconcave_chi} relates two parameters of the measurement process, the condition number of $A$ (equal to $\sqrt{\kappa}$ in the above definition) and the signal-to-noise ratio in terms of $\|Q\|$, with a geometric property of the prior, the susceptibility $\chi_t(\pi)$. 

\paragraph{Controlling $\chi_t(\pi)$} The susceptibility is not generally explicit, and may not exist for general distributions in $\mathcal{P}_2(\R^d)$.\footnote{In contrast to the \emph{expected} covariance in the Stochastic Localization framework, which contracts thanks to the martingale property \cite[Equation 11]{eldan2020taming}.}
One can nevertheless consider a simple sufficient condition from \cite{ma2019sampling}, that guarantees that $\chi_t(\pi) < \infty$ for all $t$, capturing several representative high-dimensional settings: 
\begin{proposition}[Sufficient Condition for Finite Susceptibility]
\label{prop:sufficient_cond}
If $\pi \in \mathcal{P}_2(\R^d)$ is such that $\pi = e^{-f}$, with $f \in C^1$ $m$-strongly convex outside a ball of radius $R$, and $\nabla f$ is $L$-Lipschitz, then $\chi_t(\pi) \leq (m/2+t)^{-1} e^{16 L R^2}$ for all $t$. 
\end{proposition}
This sufficient condition, proved in \Cref{app:proof_sufficient}, thus imposes both concentration, obtained here through strong log-concavity (up to perturbation), and smoothness. 
These conditions are also (jointly) necessary, as illustrated by the following counter-examples, proved in \Cref{app:proof_counter}:
\begin{proposition}[Susceptibility Blow-up]
\label{prop:example_blowup}
    There exists $\pi \in \mathcal{P}_2(\R)$ such that $\chi_t(\pi) = \infty$ for any $t >0$. Moreover, there exists $\pi \in \mathcal{P}_2(\R)$ with subgaussian tails, i.e., $\mathbb{P}_\pi( |Y| > z) \lesssim e^{-\lambda z^2}$, such that $\chi_t(\pi) = \infty$ for any $t >\lambda$. 
\end{proposition}

We can additionally establish simple, yet useful, properties of the susceptibility (proved in \Cref{app:proof_chiprop}):
\begin{proposition}[Properties of $\chi_t(\pi)$]
\label{ex:chiprop}
We have the following:
\begin{enumerate}[label=(\roman*)]
    \item Tensorization: If $\mu = \mu_1 \otimes \mu_2 \dots \otimes \mu_d$, then $\chi_t(\mu) = \max_i \chi_t(\mu_i)$. 
    \item Asymptotic behavior: If $\pi = e^f$ and $\nabla f$ is Lipschitz, then $\chi_t(\pi) = 1/t + o(1/t)$ as $t \to \infty$. 
\end{enumerate}
\end{proposition}
Finally, we conclude with explicit examples that will be used later: 
\begin{example}[Examples of $\chi_t(\pi)$]
\label{ex:chiexamples}
We have the following:
\begin{enumerate}[label=(\roman*)]
    \item Gaussian measure: $\chi_t(\gamma_d) = \frac{1}{1+t}$. 
    \item Compactly Supported Gaussian Mixture: If $\mu$ is compactly-supported in a ball of radius $R$ and $\delta\geq 0$, then $\chi_t( \mu \ast \gamma_\delta) \leq  \left( \frac{R}{1+\delta t}\right)^2 + \frac{\delta}{1+\delta t}$.
    \item Uniform measure on hypercube: If $\pi$ is uniform on the hypercube $\mathcal{H}_d$, then $\chi_t(\pi) = 1$. 
 \end{enumerate}
\end{example}
In light of these simple properties of the susceptibility, we can already extract useful information out of \Cref{thm:strong_logconcave_chi}: in the regime where $\|Q\| \ll 1$, corresponding to the low SNR setting, and
under compact support assumptions, the LHS converges to a finite value, while the RHS diverges, leading to log-concavity of $\nu_T$. On the other hand, in the regime where $\|Q \| \gg 1$, corresponding to the high SNR setting, under minimal regularity assumptions we will have $\chi_{\|Q\|}(\pi) \simeq \|Q\|^{-1} < \|Q\|^{-1} \frac{\kappa}{\kappa-1}$, leading also to a sampling guarantee.

\subsection{Comparisons}
\paragraph{With Langevin dynamics} As introduced above, Langevin dynamics and its discretized version, Langevin Monte Carlo (LMC)~\cite{roberts1996exponential,ma2015complete} serve as natural baselines for efficient posterior sampling. As such, to assess whether sampling from $\nu_T$ is easier than sampling from $\nu$, ideally one would like to compare the LSI constants $\rho(\nu)$ and $\rho(\nu_T)$ associated respectively with $\nu$ and $\nu_T$. 

While this direct comparison is not available in general, one can resort to comparing lower bounds. The starting point is to compare conditions for log-concavity, which imply lower bounds for $\rho$ via Bakry-Emery. At low SNR regimes where $\|Q \| \ll 1$, \Cref{thm:strong_logconcave_chi} and \Cref{ex:chiexamples} show that $\nu_T$ becomes log-concave under mild assumptions on $\pi$. On the other hand, $\nu$ will be generally non-log-concave in this regime, since the Hessian $\nabla^2 \log \nu = -Q + \nabla^2 \log \pi$ will converge to the prior. 

Alternatively, by \Cref{rem:polchinsky}, we can relate lower bounds for $\rho(\nu)$ and $\rho(\nu_T)$ via the multiscale Bakry-Emery criterion obtained with the Polchinksy flow using a specific covariance decomposition. Indeed, by setting $q= \|Q\|^{-1}$, $\dot{C}_t = \mathrm{Id}$ for $0 \leq t \leq  q$, and 
$\dot{C}_t = ( Q^{-1} - q \mathrm{Id})^{-1} \mathbf{1}( q < t \leq q + 1)$, from \cite[Theorem 3.6, Remark 3.7]{bauerschmidt2023stochastic} (noting the difference in time parameterization by a factor of 2 between \cite{bauerschmidt2023stochastic} and our dynamics in the heat semigroup) we obtain the following lower bounds on $\rho(\nu)$ and $\rho(\nu_T)$:
\begin{align}
\label{eq:lsi_polch}
    \rho(\nu) &\geq \frac12 \left( \int_0^{q + 1} e^{-2\lambda_t} \dd t \right)^{-1}~,~~     \rho(\nu_T) \geq \frac12  \left(e^{2 \lambda_q} \int_{q}^{q + 1} e^{-2\lambda_t} \dd t \right)^{-1}~,
\end{align}
with $\lambda_t = \int_0^t \dot{\lambda}_s \dd s$ and $(\dot{\lambda}_t)_t$ any sequence satisfying, for $t \in (0, q+1)$, 
\begin{align}
    \dot{C}_t \nabla^2 \log (\pi \ast \gamma_{C_t}) \dot{C}_t &\succeq \dot{\lambda}_t \dot{C}_t~.
\end{align}
Now, observe that for $t \in (0, q)$, the optimal choice for $\dot{\lambda}_t$ is simply $\inf_x \lambda_{\min}\left[\nabla^2 \log( \pi \ast \gamma_{t}(x))\right]$. Thus, if the prior is such that 
\begin{equation}
\label{eq:lsi_polch2}
\inf_x \lambda_{\min}\left[\nabla^2 \log( \pi \ast \gamma_{t} (x))\right] \leq 0 ~~\text{for }~~t \in (0,q)~,    
\end{equation}
we have $\lambda_q <0$, and therefore
\begin{align}
 \left(e^{2 \lambda_q} \int_{q}^{q + 1} e^{-2\lambda_t} \dd t \right)^{-1} & \geq \left(\int_{q}^{q + 1} e^{-2\lambda_t} \dd t \right)^{-1} \\
 &\geq \left(\int_{0}^{q + 1} e^{-2\lambda_t} \dd t \right)^{-1}~,
\end{align}
showing that the lower bound for $\rho(\nu_T)$ in \eqref{eq:lsi_polch} dominates that of $\rho(\nu)$. 
The condition (\ref{eq:lsi_polch2}) describes the generic setting where the heat semigroup, run for $t \in (0, \|Q\|^{-1})$, is not sufficient to make the prior measure log-concave, as the latter requires a much stronger condition:
\begin{equation}
\exists~t \in (0, q) \text{~~such that~~}\sup_x \lambda_{\max}\left[\nabla^2 \log( \pi \ast \gamma_{t} (x))\right] \leq 0~.
\end{equation}
We emphasize however that the lower bound for $\rho(\nu)$ in (\ref{eq:lsi_polch}) might not be tight for this particular choice of covariance decomposition, and that often one needs to adapt the decomposition to the specific model.

\paragraph{With Importance Sampling} In the low SNR regime with a well-conditioned $A$, the posterior measure can be viewed as a small perturbation of the prior. As such, a natural baseline for posterior sampling is Importance Sampling (IS), using the prior as a proposal --- for which samples can be efficiently obtained thanks to the denoising oracle. In the low SNR regime, one thus expects the variance of the sampled weights to be small. 

However, we now argue that, while IS is a nature baseline in this low SNR regime, it generally suffers from exponential complexity when the SNR is high. 
In order to estimate an integral of a function $f$ with respect to the posterior measure $\nu$
$$
I(f) \coloneqq \int_{\R^d} f(x)\dd \nu(x),
$$
the idea of importance sampling is to independently sample $X_1, \dots, X_n$ from the prior $\pi$ and calculate
$$
I_n(f) \coloneqq \frac{\sum_{i=1}^n f(X_i)\tau(X_i)}{\sum_{i=1}^n \tau(X_i)},
$$
where $\tau(x)$ is the observation likelihood $\exp\left(-\frac12 x^\top Q x + x^\top r\right)$. With $\mathsf{DO}_\pi$, we can sample from the prior efficiently. %
When $\sigma$ is large such that the ratio $\tau$ is close to 1, $I_n(f)$ computed from prior samples can efficiently approximate $I(f)$. On the contrary, if $\sigma$ is small, $\tau(x)$ can have very large variance and the importance sampling can be inefficient since many prior proposals have very small weights. The work \cite[Theorem 1.2]{chatterjee2018sample} proves that, in a fairly
general setting, a sample of size approximately $\exp(\mathrm{KL}(\nu||\pi))$ is necessary and sufficient for accurate estimation by importance sampling, where $\mathrm{KL}(\nu||\pi)$ is the Kullback–Leibler divergence of $\pi$ from $\nu$:
$$
\mathrm{KL}(\nu||\pi) = \int_{\R ^d} \log \left(\frac{\dd \nu}{\dd \pi}\right) \dd\nu = \int_{\R ^d} \left(-\frac12 x^\top Q x + x^\top b\right) \dd\nu(x)~= -\frac12 \langle \Sigma, Q\rangle + c^\top b~,
$$
where $\Sigma = \E_\nu [xx^\top]$ and $c = \mathbb{E}_\nu[x]$ are the first two moments of $\nu$. 
This result confirms one part of the intuition above: if $\sigma$ is sufficiently large, then the magnitude of $Q$ and $r$ will be sufficiently small, and so is $\mathrm{KL}(\nu||\pi)$ and the number of samples needed in the importance sampling. Next we show that for a fairly generic prior distribution $\pi$, when the SNR is large, $\mathrm{KL}(\nu||\pi)$ will be also large such that we need approximately $\mathcal{O}(e^{d\cdot \mathrm{SNR}})$ examples to implement importance sampling, which is unachievable. 

Without loss of generality, we assume the covariance of the prior $\pi$ is $\mathrm{Id}$.

\begin{proposition}[Importance Sampling Sample Complexity Lower Bound]
\label{prop:ISlower}
    Assume $\nabla \log \pi(x)$ is $L$-Lipschitz:
\begin{equation}
\label{eq:ISass1}
\|\nabla \log \pi(x)- \nabla \log \pi(z)\| \leq L \|x - z \|~. 
\end{equation} 
Then, when $\SNR > L + 2$, we have
    \begin{equation}
         \mathrm{KL}(\nu||\pi) \geq \mathcal{O}(\exp(d \cdot \SNR))~,
    \end{equation}
    and therefore the sample complexity of IS is exponential in dimension. 
\end{proposition}
The proof of \Cref{prop:ISlower} is in \Cref{app:is_exponential}, and leverages Talagrand's transport inequality to relate the KL divergence term to the Wasserstein distance $W_2^2( \nu, \pi)$, which can be effectively bounded in the regime of high SNR.

\subsection{Stability Analysis}
\label{sec:stability_proof}
In the numerical implementation of boosted posterior, we typically encounter certain errors. Especially, we may have imperfect score subject to certain $L^2$ errors, and we may not be able to sample the boosted posterior $\nu_t$ exactly. Suppose that instead of starting from $\nu_t$ at $t$ and run the exact reverse SDE~\eqref{eq:reverse_sde}, we start from an approximate distribution $q_t \approx \nu_t$ and run the reverse SDE~\eqref{eq:reverse_sde} with approximating score~$s_{\theta}(x, t) \approx \nabla \log {\pi_t}(x)$ where $\theta$ denote the parameters parametrizing the score. Denote the distribution of the final samples by $q_0$. We have the following error estimate
\begin{proposition}[Stability of Tilted Transport]
\label{prop:stability}
Suppose $\nu_t, q_t, \nabla \log {\pi_t}, s_{\theta}(x, t)$ has enough regularities such that the reverse SDEs exist, if the Novikov’s condition $$\E\left[\exp(\int_{0}^t \| \nabla \log {\pi_\tau}(x) - s_{\theta}(x, \tau)\|^2 \dd \tau)\right] < \infty$$ holds, then
\begin{align}
\mathrm{KL} (\nu||q_0) \leq  \int_{0}^t\E_{\nu_\tau} \| \nabla \log {\pi_\tau}(x) - s_{\theta}(x, \tau)\|^2 \dd \tau + \mathrm{KL} (\nu_t || q_t).
\end{align}
\end{proposition}
The above proposition ensures that if both the initialization error and score $L^2$ error (over the posterior paths) are small, then the distribution of our final samples is close to the target posterior. Note that the score $L^2$ error in the RHS is \emph{not} a Fisher divergence, since we are estimating the error using the posterior $\nu_\tau$ rather than the prior $\pi_\tau$. The stability guarantee in \Cref{prop:stability} thus requires a Denoising Oracle robust to Out-of-Distribution errors. That said, this OOD robustness appears to be unavoidable in our setting of posterior sampling. 

The proof is provided in~\Cref{app:stability_proof}. Note that we consider the reverse dynamics in continuous-time without time discretization error. There are various works~\cite{lee2022convergence,lee2023convergence,chen2023sampling,benton2024nearly} analyzing the time discretization error and those techniques can be further incorporated into the above error estimate.

\section{Case Studies}
In this section we specialize the results of \Cref{sec:quant} to representative models. Posterior distributions of the form $\nu = \mathsf{T}_Q \pi$ where $\pi$ is a product measure provide particularly explicit calculations. 
\subsection{Gaussian Mixtures}
By applying \Cref{thm:strong_logconcave_chi} to \Cref{ex:chiexamples} (ii), we directly obtain the following guarantee for generic comptactly supported Gaussian mixtures (proof in \Cref{app:gaussmixt}): 
\begin{corollary}[Tilted Transport for Gaussian Mixtures]
\label{coro:gaussmixt}
If $\pi = \mu \ast \gamma_{\delta}$ and $\text{diam}(\text{supp}(\mu)) \leq R$, then $\nu_{T}$ is strongly log-concave if
\begin{equation}
    \label{eq:necessary_cond}
        R^2 < \frac{(1+\delta\SNR^2)(\delta\kappa(A)^2+\SNR^{-2})}{\kappa(A)^2-1} ~.
    \end{equation}
It also holds when $\delta=0$ and the prior $\pi$ is any distribution with a bounded support radius $R$.
\end{corollary}

\Cref{fig:phase_diagram} displays several contours of the condition in \cref{eq:necessary_cond} as a function of SNR and $\kappa(A)$. 
Each $U$-shaped contour is determined by a combination of $\delta$ and $R$, which uniquely characterizes the prior. For all points ($(\mathrm{SNR}), \kappa(A)$) outside of a contour, representing a specific inverse problem, ${\nu}_{T}$ is strongly log-concave and thus easy to sample.
Given an observation model where both SNR and $\kappa(A)$ are fixed, it is straightforward to see that the condition in \cref{eq:necessary_cond} is more readily satisfied as $\delta$ increases and $R$ decreases. \Cref{fig:phase_diagram} also confirms this result since as $\delta$ increases or $R$ decreases, the $U$-shaped contour shrinks and the region of easy to sample expands. 
Now we discuss the implications in the reverse scenario where the prior is fixed and the observation model is adjusted. If we look at \Cref{fig:phase_diagram} horizontally, we know that given a prior and $\kappa(A)$, the target posterior can be reliably sampled if the SNR is either sufficiently low or high, with the region of mid-SNR being challenging. The closer $\kappa(A)$ is to 1, the smaller this challenging region is. When the problem is denoising such that $\kappa(A) = 1$, the challenging region vanishes, and sampling the posterior is straightforward using the denoising oracle, as previously explained.

\begin{figure}[!ht]
\centering
\includegraphics[width=0.9\textwidth]{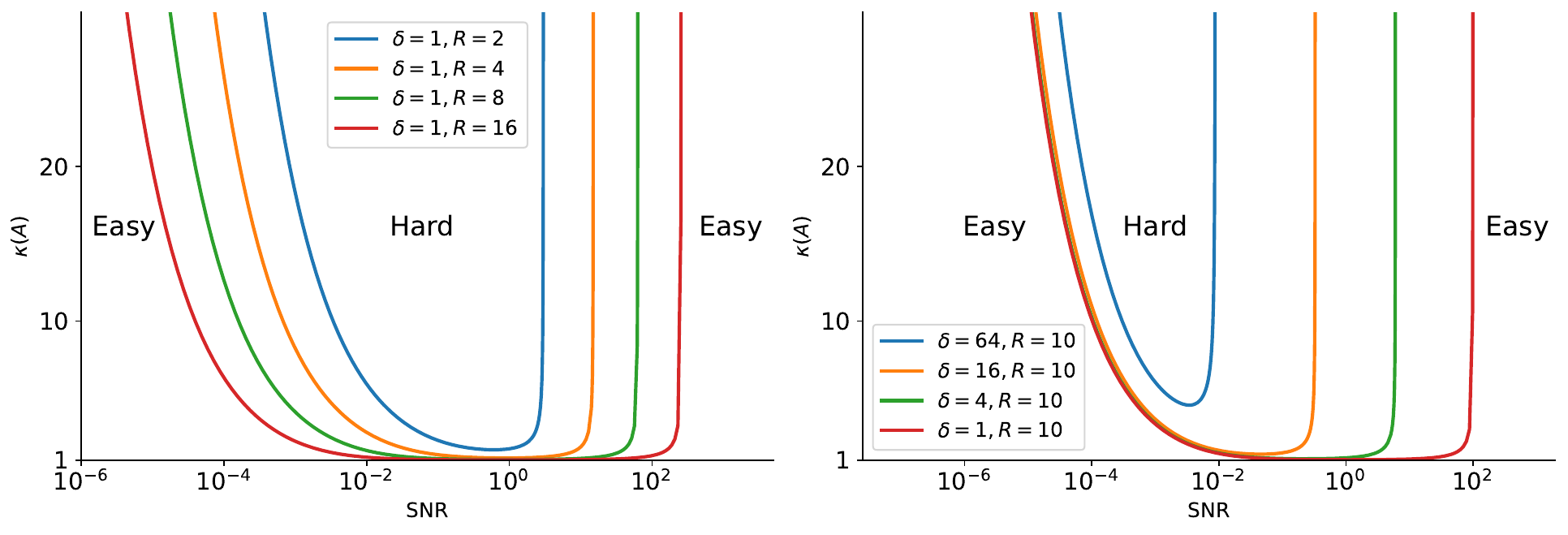}
\caption{Phase diagram for the boosted posterior ${\nu}_{T}$ being strongly log-concave in~\Cref{coro:gaussmixt}.}
\label{fig:phase_diagram}
\end{figure}

\subsection{Ising Models}
As a direct consequence of \Cref{thm:strong_logconcave_chi} and \Cref{ex:chiexamples} (iv), we  establish a sampling guarantee for Ising models: 
\begin{corollary}[Tilted Transport for the Ising Model]
\label{coro:ising}
Let $\pi$ be the uniform measure on the hypercube, and $Q$ such that $\lambda_{\max}(Q) - \lambda_{\min}(Q) < 1$. Then $\nu_{T}$ is strongly log-concave, and therefore $\nu = \mathsf{T}_Q \pi$ can be sampled efficiently (in continuous-time). 
\end{corollary}
This result thus establishes that Ising models 
may be sampled using Tilted Transport, provided their spectrum satisfies $\lambda_{\max}(Q) - \lambda_{\min}(Q) < 1$, thus precisely matching the computational lower bound of 
\cite{doi:10.1137/1.9781611977912.180}.
We remark though that our procedure is not (yet) algorithmic; a careful analysis of the discrete-time complexity and the approximation rates to avoid the singularity of the score of $\pi_t$ as $t \to 0$ (via Proposition \ref{prop:hardness_smooth}) should be formalized. 

In itself, this sampling guarantee should not come as a surprise, since, as discussed previously, Glauber dynamics are already known to mix efficiently in this regime \cite{eldan2022spectral,anari2021entropic}, and a Log-Sobolev Inequality was established even earlier than these in \cite{bauerschmidt2019very}. That said, these positive guarantees both require a `multiscale' decomposition of entropy that goes beyond the Bakry-Emery criterion, using either the framework of stochastic localization \cite{eldan2020taming,chen2022localization} or the similar Polchinsky renormalisation group framework \cite{bauerschmidt2023stochastic}. In that sense, an arguably interesting feature of our result lies in its simplicity: we only exploit the Bakry-Emery criterion at the tilted measure $\nu_T$. 

If one specializes \Cref{coro:ising} to the Sherrington-Kirkpatrick (SK) model, the equivalent inverse temperature that guarantees sampling is $\beta = 1/4$, which remains below $\beta = 1$, the threshold of the hard phase. For this threshold, dedicated AMP-based sampling succeeds \cite{el2022sampling, celentano2024sudakov}.
As future work, it would be interesting to explore %
whether the Polchinsky-based multiscale Bakry-Emery criterion could be applied to $\nu_{T}$ to improve upon \Cref{thm:strong_logconcave_chi} in this model. Similarly as in the referred prior works \cite{eldan2020taming, chen2022localization, bauerschmidt2019very}, the key ingredient that would remove the roadblock is a sharper bound on the susceptibility $\chi_t(\pi)$; and in particular extending the existing bounds on the covariance of the SK model \cite{el2024bounds, brennecke2023operator}, valid for $\beta < 1$, to arbitrary external fields.

\subsection{Scalar Field $\varphi^4$ model}

As a further illustration, we now consider the two-dimensional lattice scalar $\varphi^4$ field model, where 
$\nu$ is given by 
\begin{equation}
\label{eq:phi4_measure}
    \nu = \mathsf{T}_{\beta\Delta} {\pi_\beta} \in \mathcal{P}(\Omega)~,
\end{equation}
where $\Omega$ is a discrete two-dimensional lattice of total size $d$, $\beta$ is the inverse temperature, $\Delta$ is the discrete Laplacian on $\Omega$ and ${\pi}_\beta = \mu_\beta^{\otimes d}$ is a product 
measure, whose marginal in each site is given by 
$$d\mu_\beta(\varphi) \propto e^{-\varphi^4 + (1+ 2\beta) \varphi^2} d\varphi~.$$
The scalar potential thus has the familiar double-well profile that enforces configurations close to $\pm 1$. 

This model is known to satisfy a Log-Sobolev Inequality with $
\rho$ independent of $d$ as long as $\beta < \beta_c \approx 0.68$, by exploiting the multiscale Bakry-Emery criterion along a Polchinsky flow with carefully chosen covariance decomposition \cite{bauerschmidt2022log, bauerschmidt2023stochastic}. 
We can alternatively apply \Cref{thm:strong_logconcave_chi} directly to $\nu_T$. 
Using $\| \beta \Delta\| = 4\beta$, and $\lambda_{\min}(\Delta) = 0$, we obtain that $\nu$ can be efficiently sampled using tilted transport whenever 
\begin{equation}
    \eta < \frac{1}{4\beta}~,\text{ where }~\eta = \sup_{\alpha} \mathrm{Cov}[\mathsf{T}_{0,\alpha} \mu_0]~,
\end{equation}
where $\mathsf{T}_{0,\alpha}$ is thus a linear tilt. 
Numerically estimating $\eta$ gives $\eta \approx 0.52$ and therefore $\nu_T$ is log-concave whenever $\beta < 0.48$. This is slightly below the phase transition $\beta_c \approx 0.68$, indicating that in this example the direct Barky-Emery criterion, even if it is applied to the tilted measure, is not as sharp as the multiscale criterion. 

A natural question is thus whether one could apply the techniques of \cite{bauerschmidt2022log} to $\nu_T$. Let us briefly describe the main technical challenge that needs to be overcome for this purpose. 
We recall that the tilted posterior is $\nu_* = \mathsf{T}_{Q_*} (\pi \ast \gamma_{\|Q\|})$ (using the Heat semigroup w.l.o.g.), where $Q_* = ( \beta^{-1} \Delta^{-1} - (4\beta)^{-1} \mathrm{Id})^{-1}$. 
The covariance decomposition used to establish the sharp LSI bound for $\nu$ exploits a key structural property of the operator $Q = \Delta$, namely that it is \emph{ferromagnetic}, i.e., $Q_{ij} \leq 0$ for all $i \neq j$, enabling a sharp control of the associated susceptibility $\chi_t(\nu)$; see \cite{bauerschmidt2022log} and also \cite[Section 5.1.2]{chen2022localization} for further details. The ferromagnetic property is unfortunately lost in $Q_*$, despite being a more `convex' potential.

\section{Iterated Tilted Transport}
\label{app:iterated_transport}

We have shown that posterior sampling of $\nu = \mathsf{T}_Q \pi$ can be 
reduced to sampling from $\nu_{T}$ by running the reverse SDE. 
While $\nu_{T}$ is easy to sample under the conditions presented in Section \ref{sec:quant}, these may not be verified in several situations of interest. 
In this context, a natural question is whether one could still leverage the tilted transport, at the expense of introducing sampling error. This is what we address in this section. 

Let $\lambda_1, \ldots \lambda_d$ be the eigenvalues of $Q$.
Let us assume for simplicity that all eigenvalues have multiplicity $1$, so $\lambda_i > \lambda_{i+1}$. We also adopt the heat semigroup to simplify the exposition without loss of generality. 
We define the events 
$T_j$ for $j=1\ldots d$ given by 
\begin{equation}
T_j := \frac12 \lambda_j^{-1}~.    
\end{equation}
Denote by 
$$\bar{\lambda}_j(t) = \begin{cases}
    \infty & \text{if } t \geq T_j~, \\
    \lambda_j(t) & \text{ otherwise, }
 \end{cases}$$
where $\lambda_j(t) = \frac{\lambda_j}{1-2t \lambda_j}$ is the solution to the ODE $\dot{q}_t = 2q_t^2$ for tilted transport in the heat semigroup setting.
By abusing notation, we denote by $\bar{Q}_t$ the matrix that shares eigenvectors with $Q$, and with eigenvalues $( \bar{\lambda}_1(t), \ldots, \bar{\lambda}_d(t))$. 
Denote by $V_k = [v_{d-k} \ldots v_d] \in \R^{d \times k}$ the orthogonal projection onto the last $k$ eigenvectors.

While previously we considered only the transport between $\nu$ and $\nu_1:= \nu_{T_1}$, now we can 
consider the sequence $\nu_{k}:= \mathsf{T}_{\bar{Q}_{T_k}} \pi_{T_k}$ for $k=1, \ldots, d$, where we denote $\pi_t = \pi \ast \gamma_{2t}$ the action of the heat semigroup on $\pi$. 
Observe that $\nu_k$ is a measure supported on a subspace $\Omega_k$ of dimension $d-k$; in other words, $k$ directions are singular, corresponding to the eigenvectors associated with the $\infty$ eigenvalues of $\bar{Q}_{T_k}$, and thus $\Omega_k = \{ x \in \R^d; V_k^\top x = {\bf y}_k \}$ for some ${\bf y}_k \in \R^k$. 

Now, let us consider $k^\star = \min \{k; \nu_k \text{ is s.l.c.} \}$; that is, the first $k$ such that $\nu_k$ is strongly log-concave, and therefore efficiently sampleable by Langevin dynamics. 
By defining $\kappa_k := \frac{\lambda_k}{\lambda_d}$ as the condition number of the truncated $Q$, we immediately obtain from \Cref{thm:strong_logconcave_chi} the bound
\begin{equation}
    k^\star \leq \min\left\{k; \chi_{\lambda_k}(\pi) \leq \lambda_k^{-1} \frac{\kappa_k}{\kappa_k -1} \right\}~.
\end{equation}

For $k < k^\star$, assume first that one had sampling access to $\nu_{k+1}$. Running the reverse tilted transport for time $\eta_k:=T_{k+1} - T_{k}$ produces samples from a tilted measure $\tilde{\nu}_k:= \mathsf{T}_{\tilde{Q}_k} \pi_{T_k}$, where 
$\tilde{Q}_k$ has eigenvalues
\begin{align}
&(\tilde{\lambda}_1(T_k),\dots, \tilde{\lambda}_k(T_k), \tilde{\lambda}_{k+1}(T_k), \tilde{\lambda}_{k+2}(T_k), \dots, \tilde{\lambda}_d(T_k)) \\
:= &\left( \underbrace{\frac{\lambda_{k+1}}{1 - \lambda_k^{-1}\lambda_{k+1}}, \dots, \frac{\lambda_{k+1}}{1 - \lambda_k^{-1}\lambda_{k+1}}}_{k \text{ times }}, \frac{\lambda_{k+1}}{1 - \lambda_k^{-1}\lambda_{k+1}}, \frac{\lambda_{k+2}}{1 - \lambda_k^{-1}\lambda_{k+2}}, \dots, \frac{\lambda_{d}}{1 - \lambda_k^{-1}\lambda_{d}}\right)~.
\end{align}
One can easily verify the above fact by noting that under the ODE $\dot{q}=2q^2$ with initial condition $q(T_k)=\tilde{\lambda}_j(T_k)$, then the solution at time $T_{k+1}$ will be exactly $q(T_{k+1})=\bar{\lambda}_j(T_{k+1})$.
Since all $\tilde{\lambda}_j(T_k) <\infty$, $\tilde{\nu}_k$ is thus a non-singular measure in $\R^d$, capturing the fact that the Brownian motion driving the reverse dynamics is isotropic, and thus oblivious to the existence of the singular support of $\nu_k$.

We thus need a procedure to transform samples from $\tilde{\nu}_k$ to approximate samples of $\nu_k$. 
The simplest way is to simply marginalize the coordinates $(x_1, \ldots, x_k) = V_k^\top x \in \R^k$, 
i.e., $$\bar{\nu}_k(x_{k+1}, \ldots, x_d) = \int_{\R^k} \tilde{\nu}_k( \dd x_1, \ldots, \dd x_k, x_{k+1}, \ldots, x_d) \in \mathcal{P}( \R^{d-k})~,$$ and then `lift' this measure in the subspace $\Omega_k$, i.e., 
\begin{align}
    \hat{\nu}_k( {\bf x}_k; {\bf x}_{-k}) &:= \delta( {\bf x}_k - {\bf y}_k) \bar{\nu}_k ( {\bf x}_{-k})~,
\end{align}
where we defined ${\bf x}_k = (x_1, \ldots, x_k) $ and ${\bf x}_{-k}=(x_{k+1}, \ldots, x_d)$. Under this definition, given a sample from $\nu_k$, we just project its $k$ components to $\bm{y}_k$ to get a sample from $\hat{\nu}_k$, as an approximation to the sample from $\nu_k$.

\begin{algorithm}
\caption{Sampling Using Iterated Tilted Transport}
\label{alg:iterated_tilt}
\begin{algorithmic}[1]
    \State Start by sampling $X_{k^\star} \sim \nu_{k^\star}$. 
    \For{$k = k^\star - 1$ \textbf{to} $0$}
        \State Run tilted transport starting at $X_{k+1}$ for time $\eta_k$, resulting in $\tilde{X}$.
        \State Set $X_k = ( \mathbf{y}_k; \mathbf{\tilde{X}}_{-k})$. \label{alg:line:xk_set}
    \EndFor
    \State \textbf{return} $X_0$
\end{algorithmic}
\end{algorithm}

\begin{figure}[!ht]
\centering
\includegraphics[width=0.6\textwidth]{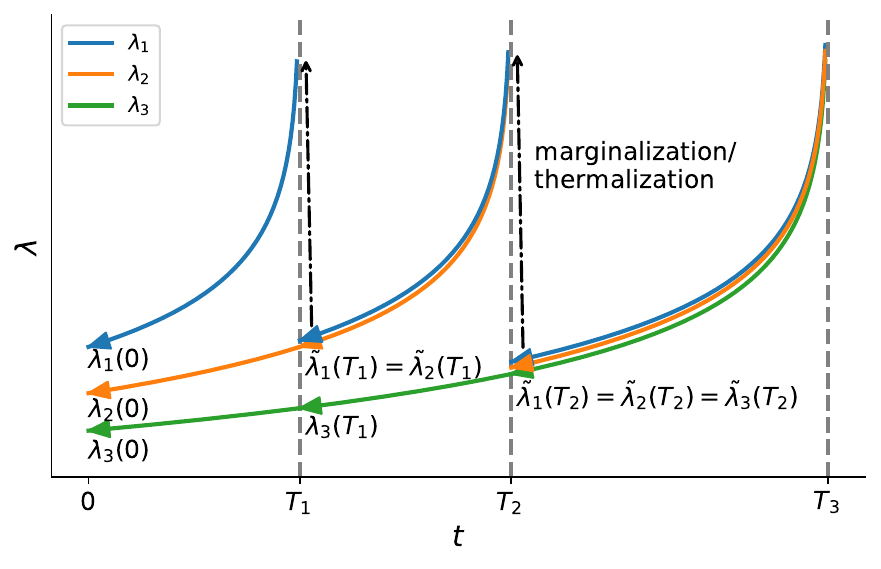}
\caption{Scheme plot of iterated titled transport in an example with $d=3, k^\star=3$. The tilted transport from $T_{k+1}$ to $T_k$ transforms samples of $\nu_{k+1}$ (corresponding to eigenvalues $\bar{\lambda}_j(T_{k+1})$ to $\tilde{\nu}_k$ (corresponding to eigenvalues $\tilde{\lambda}_j(T_k)$), and marginalization/thermalization (denoted by the dashed lines) further transforms samples of $\tilde{\nu}_k$ toward ${\nu}_k$.}
\label{fig:iterated_scheme}
\end{figure}

We can then iteratively run the tilted transport backwards, from $k=k^\star-1$ to $k=0$, as illustrated in \Cref{alg:iterated_tilt} and \Cref{fig:iterated_scheme}.
By the data-processing inequality, the TV error will accumulate linearly at each step. Denoting $\hat{\nu}$ the law of $X_0$, we have
\begin{equation}
    \mathrm{TV}(\hat{\nu}, \nu) \leq \sum_{0<k < k^\star} \mathrm{TV}( \hat{\nu}_k, \nu_k)~.
\end{equation}
This bound can be interpreted as the accumulation of errors arising from conditioning a measure by marginalizing over its first $k$ components. To the extent that $\frac{\lambda_{k+1}}{\lambda_k} \approx 1$ such that $\eta_k \ll 1$, these variables are nearly deterministic, so one would expect that marginalization is a good approximation of conditioning. The outstanding question is to understand conditions when this error guarantee can be quantified. 

Inspired by \cite{chen2024probability}, a natural extension of this simple iterative procedure based on marginalization is to apply `thermalization' towards the stationary measure $\nu_k$ after line~\ref{alg:line:xk_set} of \Cref{alg:iterated_tilt} above, by running Langevin dynamics in $\Omega_k$ with score $\nabla \log \nu_k$:
\begin{align}
    \dd X_t &= \nabla \log \nu_k(X_t) \dd t + \sqrt{2} \dd W_t~,~X_0 \sim \hat{\nu}_k.
\end{align}
The drift of this diffusion is available, since both $\bar{Q}_{T_k}$ and $\nabla \log \pi_{T_k}$ are known, so is $\nabla \log \nu_k$. 

Denote by $\check{\nu}_k$ the law of $X_t$ after time $t = B_k$. 
While the time to relaxation of such Langevin dynamics is generally not quantitative (otherwise $k^\star \leq k$), even a short amount of thermalization is able to improve upon the previous method. Indeed, by the  reverse transport inequality \cite[Lemma 4.2]{bobkov2001hypercontractivity}, a weaker Wasserstein guarantee $W_2( \hat{\nu}_k, \nu_k)$ can be `upgraded' to a TV guarantee of the form $\mathrm{TV}(\check{\nu}_k, \nu_k  ) = O( \sqrt{L_k} W_2(\hat{\nu}_k, \nu_k))$ by running Langevin dynamics for time $B_k = \Theta(1/L_k)$, 
where $L_k = \sup_x \lambda_{\max}( \nabla^2 \log \nu_k(x) ) > 0$ is the largest eigenvalue of $\nabla^2 \log \nu_k$, which is positive by definition of $k^\star$ and $k<k^\star$. Notice also that from 
\Cref{eq:basic} we have the upper bound 
$$
    L_k \leq (1 +\lambda_k) \left( \lambda_k \chi_{\lambda_k}(\pi) - \frac{\kappa_k}{\kappa_k-1}  \right) ~. $$   
In summary, the `thermalized' iterated tilted transport satisfies an error bound of the form
\begin{equation}
\label{eq:it_thermal}
    \mathrm{TV}(\hat{\nu}, \nu) \lesssim \sum_{0<k < k^\star} \sqrt{L_k} W_2(\hat{\nu}_k, \nu_k)~.
\end{equation}
We can interpret the error guarantee of \Cref{eq:it_thermal} as being able to leverage a warm-start in Wasserstein distance when running Langevin dynamics at each step $k$. In that sense, it would be interesting to understand whether the further structure of the warmstart could be further exploited.

\paragraph{Interpretation as a Homotopy `Frequency-Marching' Method}
If one considers an inverse problem 
where the data $x$ is a signal over a physical domain, e.g. an image or a time-series, 
and the measurement operator $A$ is a translation-invariant low-pass filter, then $Q=A^\top A$ diagonalises in the Fourier basis, and the eigenvalues of $Q$ are sorted from low to high-frequencies. In this prototypical setting, the iterative tilted transport \Cref{alg:iterated_tilt} can be viewed as 
an homotopy method for sampling, whereby the low-frequency projection of the measurements (encoded in the vector ${\bf y}_k$) is used to condition the reverse diffusion process; this is reminscient of `Frequency-Marching' methods, a powerful heuristic used in several imaging inverse problems \cite{chen1997inverse,barnett2017rapid}.

\section{Numerical Experiments}

\paragraph{Gaussian Mixture Model}
We verify our theoretical results using the Gaussian mixture model in high dimensions. Same to the models considered in \cite{cardoso2024monte}, the prior distribution is a mixture of 25 components with known means and variances (see \Cref{fig:schemeplot} for a 2D visualization and \Cref{app:exp_details} for detailed settings). We examine three cases where $d=20, 40$, and $80$. In each scenario, we set $d'=d$, fix $\kappa=20$, and vary the SNR from $10^{-5}$ to $10^{-1}$. 
We use the Sliced Wasserstein distance as a principled error metric, computed from samples obtained by our algorithms and samples directly from the analytically computed posterior Gaussian mixture. 
\Cref{fig:gmm_err} illustrates the comparison between LMC and LMC boosted by tilted transport. As analyzed earlier, LMC is effective when the SNR is high enough to render the target posterior strongly log-concave, but its error quickly increases as the SNR decreases. In contrast, the tilted transport enhances LMC to perform well in both low and high SNR regimes with small sampling errors. Its performance is weaker in the mid-SNR regime compared to the extremes, as predicted by \Cref{coro:gaussmixt}. However, the tilted transport still improves upon LMC in this challenging regime by boosting effective SNR and simplifying the prior.

\begin{figure}[!ht]
\centering
\includegraphics[width=0.9\textwidth]{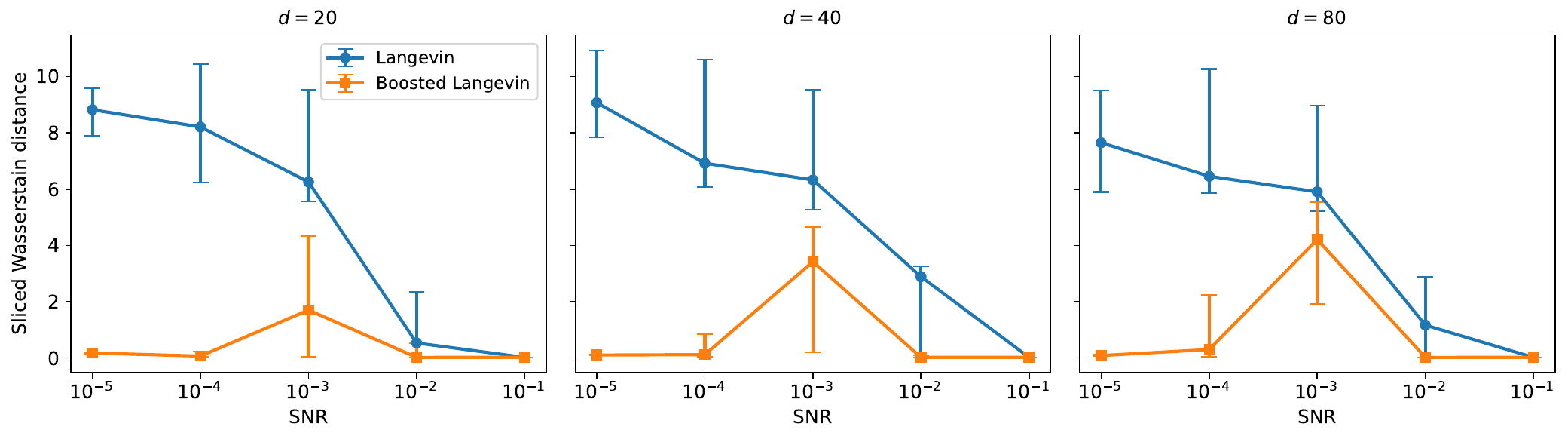}
\caption{Comparison of Langevin and boosted Langevin for Gaussian mixture prior.
We generate the prior, measurement and sample the posterior under 20 different instances in each setting. The sliced Wasserstein distances are plotted with the median in the middle, and the 25th and 75th percentiles indicated by the error bars.}
\label{fig:gmm_err}
\end{figure}

We further test tilted transport when $d' < d$, in which $\lambda_{\min}(Q_t)$ remains zero but the signal corresponding to the non-zero eigenvalues still gets enhanced. Therefore, although it becomes more difficult for $\nu_{T}$ to be strongly log-concave, the tilted transport can still make the new posterior easier to sample even if it is not strongly log-concave yet. Detailed results are reported in~\Cref{app:gmm_results}.

\paragraph{Scalar Field $\varphi^4$ model}
We test the effect of tilted transport on the scalar field $\varphi^4$ model. We run the Langevin dynamics for the target measure $\nu$ in \Cref{eq:phi4_measure} and its tilted version near blowup time, using a lattice size of $128 \times 128$. To compare the relaxation time speed, we compute the autocorrelation of each single site over 100,000 steps after the burn-in stage. \Cref{fig:phi4_acf} shows the autocorrelation averaged over all sites for $\beta = 0.2, 0.4$, and 0.6. We clearly observe that the relaxation becomes slower as $\beta$ increases, but in all three cases, the tilted transport accelerates the Langevin dynamics.

\begin{figure}
\centering
\includegraphics[width=0.9\textwidth]{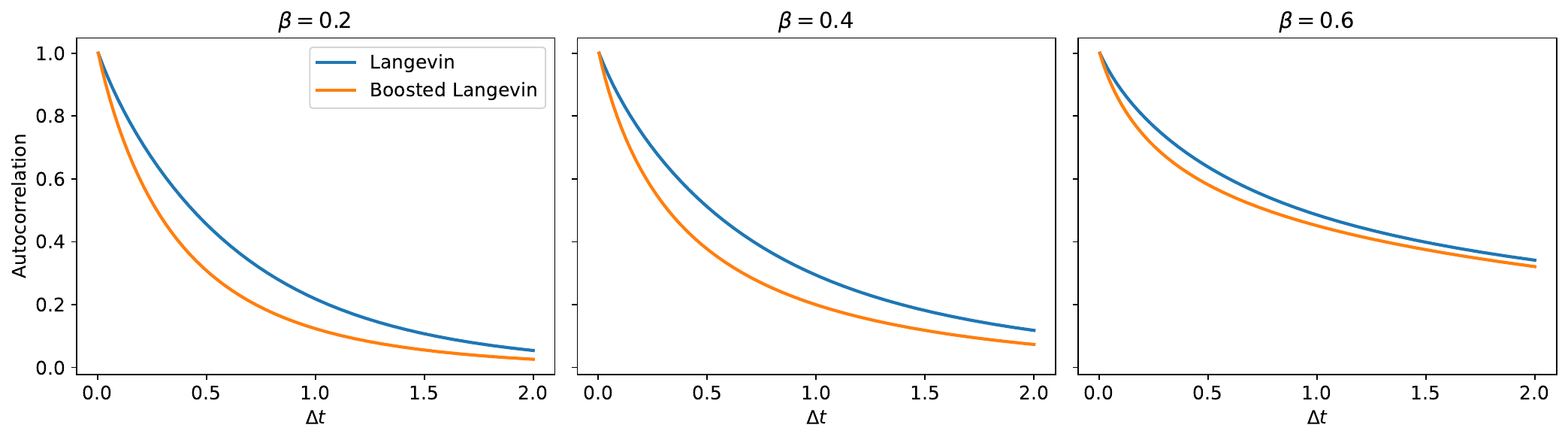}
\caption{Comparison of autocorrelation in Langevin dynamics between the original posterior and boosted posterior distributions. The autocorrelation is computed for each site individually and then averaged.}
\label{fig:phi4_acf}
\end{figure}

\section{Discussion and Future Work}
In this paper, we theoretically investigate posterior sampling using powerful priors provided by denoising oracles. We demonstrate that efficient posterior sampling can be challenging even with a perfect denoising oracle for the prior. To achieve provable posterior sampling, one must either constrain the measurements or leverage the structural properties of the prior. In this paper, we focused on the former, showing that well-conditioned measurements enable the proposed tilted transport technique to simplify the task significantly, providing a clear, verifiable condition for efficient sampling, as demonstrated on the Ising model. 

That said, several questions remain open: Can this approach provably handle poorly-conditioned measurements, such as inpainting? Can it be extended from linear to nonlinear inverse problems? %
We showed in \Cref{app:iterated_transport} how to extend the tilted transport beyond the condition of \Cref{thm:strong_logconcave_chi} via `iterated tilts', at the expense of introducing approximation errors. A natural goal is to quantify these errors in practical situations. 
On the theory side, the key object underlying the success of the tilted transport is the susceptibility $\chi_t(\pi)$; in particular, understanding when one can remove dimension dependence is an interesting question.  We also aim to systematically evaluate the empirical performance of tilted transport in imaging and scientific computing. ~\Cref{app:image_results} provides a proof-of-concept for various imaging tasks. We suspect that tilted transport could even improve existing posterior point estimate methods by boosting SNR and enabling proper uncertainty quantification through the reverse process.
\bibliographystyle{alpha}
\bibliography{refs}

\newpage
\input{appendix}

\end{document}

%% file: appendix.tex
\appendix

\section{Proof of Proposition \ref{prop:hardness_smooth}}
\label{app:proof_smooth}

Let $\delta>0$ and $\bar{\pi}$ be the uniform measure in the $d$-dimensional hypercube. Consider a Gaussian mixture prior $\pi$  
defined as $\pi = \bar{\pi} \ast \gamma_\delta$.

Since both $\bar{\pi}$ and $\gamma_\delta$ are product measures, it follows that $\pi$ is also a product measure, and therefore its denoising oracle $\mathrm{DO}_\pi$ is explicitly given by $\mathrm{DO}_\pi( y, t)_i = \psi( y_i, t)$, with 
\begin{align}
\psi(v, t) &= \int_{\R} u q_{v,t}(u) du ~,   \\
q_{v,t}(u) &= Z^{-1} \left( e^{-\frac12 (\delta^{-2}(u-1)^2 + t^{-2}(v-u)^2 ) } + e^{-\frac12 (\delta^{-2}(u+1)^2 + t^{-2}(v-u)^2 ) } \right)~. 
\end{align}
Observe that $q_{v,t}$ is the density of a Gaussian mixture in $\R$ of the form 
$\alpha \mathcal{N}(b_{-}, \sigma) + (1-\alpha) \mathcal{N}( b_{+}, \sigma)$, with parameters 
\begin{align}
    \sigma^{-2} &= \delta^{-2} + t^{-2} \\
    b_{\pm} &= \frac{\pm \delta^{-2} + t^{-2} v}{\sigma^{-2}} \\
    \alpha &= \frac{e^{\frac{(\delta^{-2} + t^{-2}v )^2}{2\sigma^2}} }{e^{\frac{(\delta^{-2} + t^{-2}v )^2}{2\sigma^2}} + e^{\frac{(-\delta^{-2} + t^{-2}v )^2}{2\sigma^2}}}~,
\end{align}
and thus $\psi(v, t) = \alpha b_- + (1-\alpha) b_{+}$. 

Let us now denote by $\mu_Q$ the target Ising model, supported in the $d$-dimensional hypercube, and define the approximation  $\mu_Q^\sigma := \mathsf{T}_{Q} \pi$. 
Suppose that there is an algorithm $\mathcal{A}$ that leverages the denoising oracle of $\pi$ that can efficiently sample 
from $\mu_Q^\sigma$: its law $\hat{\mu}$ satisfies $\mathrm{TV}( \mu_Q^\sigma , \hat{\mu}) \leq \epsilon$ with runtime polynomial in $d$ and $\log(\epsilon^{-1})$.

Let now $R(x) = \sign(x)$, and consider the sampler $R \circ \mathcal{A}$, which is now supported in the hypercube. 
By the triangle and data-processing inequalities, we directly have 
\begin{align}
\mathrm{TV}( R_\# \hat{\mu}, \mu_Q) &\leq \mathrm{TV}(R_\# \hat{\mu}, R_\#\mu_Q^\delta)  + \mathrm{TV}(R_\#\mu_Q^\delta, \mu_Q )\\
&\leq \mathrm{TV}( \hat{\mu},\mu_Q^\delta )  + \mathrm{TV}(R_\#\mu_Q^\delta, \mu_Q )  \\ 
&\leq \epsilon + \mathrm{TV}(R_\#\mu_Q^\delta, \mu_Q ).
\end{align}
It remains to bound the second term in the RHS. 
We have to compare two measures in the hypercube. For $\sigma \in \mathcal{H}_d:=\{\pm 1\}^d$, they are given respectively by 
\begin{align}
    \mu_Q(\sigma) &= \frac{1}{Z} e^{-\frac12 \sigma^\top Q \sigma} ~, \\
    R_\#\mu_Q^\delta(\sigma) &= \frac{1}{\tilde{Z}} \sum_{z \in \mathcal{H}_d} \int_{R(x) = \sigma}  e^{-\frac12 (x^\top Q x + \delta^{-2} \| x- z\|^2)} dx~.
\end{align}
Applying the Laplace approximation into each integral we obtain, as $\delta \to 0$,  
\begin{align}
    \int_{R(x) = \sigma}  e^{-\frac12 (x^\top Q x + \delta^{-2} \| x- z\|^2)} dx &= \begin{cases}
        \sim C_{d,\delta} e^{-\frac12 \sigma^\top Q \sigma} & \text{ if } z = \sigma~,\\
        \sim C_{d,\delta} e^{-\frac12 (\sigma\oplus z)^\top Q (\sigma\oplus z)} e^{-\frac{|\sigma-z|}{2\delta^2}} & \text{ otherwise }~,
    \end{cases}
\end{align}
where $\sigma \oplus z$ is the XOR, and $| \sigma - z| $ is the Hamming distance. 
We thus have, for any $\sigma \in \mathcal{H}_d$, 
\begin{align}
   \left| \tilde{C} R_\#\mu_Q^\delta(\sigma) - e^{-\frac12 \sigma^\top Q \sigma} \right| \leq 2^d e^{d \lambda_{\min}(Q)/2} e^{-\frac12 \delta^{-2}} \\
   \leq e^{-\frac12 \sigma^\top Q \sigma}  2^d e^{d (\lambda_{\min}(Q)+\lambda_{\max}(Q))/2} e^{-\frac12 \delta^{-2}}~.
\end{align}
It follows that we can write $R_\#\mu_Q^\delta(\sigma)$ as 
$$R_\#\mu_Q^\delta(\sigma) = C (e^{-\frac12 \sigma^\top Q \sigma} + \eta_\sigma)~, $$
with a relative error 
\begin{align}
    \frac{| \eta_{\sigma}|}{e^{-\frac12 \sigma^\top Q \sigma}} &\leq e^{d( 1 + \frac12(\lambda_{\min}(Q)+\lambda_{\max}(Q))) -\delta^{-2}/2} := \theta~.
\end{align}
It follows that
\begin{align}
    \mathrm{TV}(R_\#\mu_Q^\delta, \mu_Q ) &= O( \theta)~,   
\end{align}
and thus if $\delta \ll \frac{1}{\sqrt{d}}$, we have a negligible TV approximation. 

\section{Proofs of Section \ref{sec:tilttrans}}
\subsection{Proof of \Cref{thm:density_equivalence}}
\label{app:proof_ODE_Qb}
\begin{proof}
We denote the time-dependent score function $\nabla \log \pi_t(x)$ by $s_t(x)$. As derived in~\Cref{sec:preliminary}, if we initialize $X_{\tau}$ according to density $\rho_{\tau}$ and run the reverse SDE~\cref{eq:reverse_sde}, the density of $X_t$ for $t\leq \tau$, denoted by $\rho_t$, satisfies the backward PDE:
\begin{equation}
    \partial_t \rho_t = \nabla \cdot ((x+2s_t) \rho_t) - \Delta \rho_t.
\end{equation}
We need to verify that $\nu_t$ %
satisfies the above PDE.%
Note that a general positive function $\rho_t$ satisfies this PDE is equivalent to that $h_t=\log \rho_t$ satisfies the following Hamilton-Jacobi PDE
\begin{equation}
    \partial_t h_t = d + 2\nabla \cdot s_t + \nabla h_t \cdot (x+2s_t) - (\Delta h_t + \|\nabla h_t \|^2).
\end{equation}
By definition, we know $h_t = \log \pi_t$ satisfies the above PDE, and we need to prove $h_t = \log \nu_t = \log \pi_t - \frac12 x^\top Q_t x + x^\top b_t + F(t)$ satisfies this PDE as well. Here $F(t)$ denotes the normalizing constant. Taking the difference between two equations, we need
\begin{align}
    &-\frac12 x^\top \dot{Q}_t x + x^\top \dot{b}_t + \dot{F} = (-Q_tx+b_t)\cdot (x+ 2s_t) + \text{trace}(Q_t) + \| s_t \|^2 - \| s_t - Q_t x + b_t \|^2 \\
    \Leftrightarrow~~
    &-\frac12 x^\top \dot{Q}_t x + x^\top \dot{b}_t + \dot{F} = x^\top(-Q_t - Q_t^\top Q_t)x + x^\top (b_t + 2Q_t^\top b_t) + \|b_t\|^2 + \text{trace}(Q_t)
\end{align}
which can be satisfied by the ODE dynamics \eqref{eq:boost_ode}.
\end{proof}

\subsection{Derivation of Solution to \Cref{eq:boost_ode}}
\label{app:ode_solution}
\paragraph{Sanity Check for the Motivating Example}
In the denoising setting, we have $Q_0 = \frac{1}{\sigma^2}I_d, b_0=\frac{1}{\sigma^2}y$. The ODE \eqref{eq:boost_ode} has the explicit solution
$$ Q_t = \frac{e^{2t}}{1+\sigma^2 - e^{2t}} I_d.$$ Note that this solution has a finite blowup time when $1 + \sigma^2 - e^{2t} \rightarrow 0^+$, which is exactly at $T=\frac12\log(1+\sigma^2)$, as derived in the main text by matching the SNR. As $t\rightarrow T$, $Q_t \rightarrow \infty I_d$, 
$$\nu_t = \exp\left(f_t(x)-\frac12 x^\top Q_t x + x^\top b_t + F(t)\right) \rightarrow \mathcal{N}(Q_t^{-1}b_t, Q_t^{-1}).$$
To see the limit of $Q_t^{-1}b_t$, we only need to consider the ODE for each component since $Q$ is diagonal. So we view the above ODE as scalar ODEs. Considering the dynamics of 
\begin{equation}
\label{eq:derivation_rt}
\frac{\dd}{\dd t}\frac{r}{Q}=\frac{\dot{r}Q-r\dot{Q}}{Q^2}=-\frac{r}{Q},
\end{equation}
gives $Q_t^{-1}b_t = e^{-t}Q_{0}^{-1}b_0$.
Therefore
$$\lim_{t\rightarrow (T)^-} Q_t^{-1}b_t = e^{-T}Q_0^{-1}b_0 = e^{-T}y,$$ 
which matches the initial condition derived in the main text for the denoising case. 
Furthermore, we can explicitly verify that the intermediate distribution of $X_t$ by running the reverse SDE from $\nu_T$ is $\nu_t$:
\begin{align}
p(X_t|X_{T}=e^{-T}y) &~\propto~ p(X_t)p(X_{T}=e^{-T}y|X_t) \nonumber \\
&~\propto~\pi_t(X_t)\exp\left(- \frac12\frac{\|e^{-T}y-e^{-(T-t)}x\|^2}{1-e^{-2(T-t)}}\right) \nonumber \\
&= \pi_t(X_t)\exp\left(- \frac12\frac{\|e^{-t}y-x\|^2}{e^{2(T-t)}-1}\right)
\end{align}
To match the form of $\nu_t$, we have $Q_t=\frac{1}{e^{2(T-t)}-1}=\frac{e^{2t}}{1+\sigma^2-e^{2t}}, Q_t^{-1}b_t = e^{-t}y$, which are the solutions of the ODE~\eqref{eq:boost_ode}.

\paragraph{Solution to \cref{eq:boost_ode}} We recall that the observation operator $A \in \R^{d'\times d}$ has a general singular value decomposition form $A=U\Sigma V^\top$ with non-zero singular values $\lambda_1 \geq \lambda_2 \geq \dots \geq \lambda_{d'} > 0$.
By definition, we have
$
Q_0 = V\text{diag}(\lambda_1^2/\sigma^2, \cdots, \lambda_{d'}^2/\sigma^2, 0, \cdots, 0) V^\top. 
$
By left multiplying $V^\top$ and right multiplying $V$ to the first ODE in \eqref{eq:boost_ode}, we can diagonalize it to scalar equations $\dot{q}_t = 2(1+q_t)q_t$ for each diagonal entry. Solving this ODE gives
\begin{align}
    Q_t = V \text{diag}\left( \frac{e^{2t}}{1 + \sigma^2/\lambda_1^2-e^{2t}}, \cdots, \frac{e^{2t}}{1 + \sigma^2/\lambda_{d'}^2-e^{2t}}, 0, \cdots, 0 \right) V^\top.
\end{align}
Here we explain how to solve $b_t$ from~\cref{eq:boost_ode}. We denote $V=[v_1, \cdots, v_d]$, in which $v_i$ are eigenvectors of $Q$ (and $Q_t$ as well), and denote the eigenvalues of $Q_t$ ($0\leq t < T$) by
\begin{equation}
\tilde{\lambda}_i(t) = 
\begin{cases}
     \displaystyle{\frac{e^{2t}}{1 + \sigma^2/\lambda_i^2-e^{2t}}},  & 1 \leq i \leq d' \\
     0,   & d' + 1 \leq i \leq d
\end{cases}
\end{equation}
By definition, we know $\tilde{\lambda_i}$ satisfies the ODE
$$
\dot{\tilde{\lambda}} = 2(1+\tilde{\lambda})\tilde{\lambda}.
$$
We rewrite the solution $b_t=\sum_i^d \xi_i(t)v_i$ and aim to solve $\xi_i(t)$. From $b_0 =V( \frac{1}{\sigma^2}\Sigma^\top U^\top y)$, we have the initial condition
\begin{equation}
    \xi_i(0) = 
    \begin{cases}
        \displaystyle{\frac{\lambda_i}{\sigma^2} (U^\top y)_i},   &1\leq i \leq d' \\
        0,  & d'+1 \leq i \leq d
    \end{cases}
\end{equation}
Taking the inner product between $v_i$ and both sides of the ODE $\dot{r}_t = (I+2Q_t)b_t$, we have
$$
\dot{\xi}_i(t) = (1+2\tilde{\lambda}_i(t))\xi_i(t).
$$
Therefore, for $d'+1 \leq i \leq d$, $\xi_i(t) =0$. For $1\leq i \leq d'$, same to the derivation in~\cref{eq:derivation_rt}, we have
$$
\frac{\dd}{\dd t}\frac{\xi_i}{\tilde{\lambda}_i} = -\frac{\xi_i}{\tilde{\lambda}_i},
$$
which gives
\begin{align}
&\frac{\xi_i(t)}{\tilde{\lambda}_i(t)} = e^{-t}\frac{\xi_i(0)}{\tilde{\lambda}_i(0)}, \\
\Rightarrow~ &(\frac{e^{2t}}{1 + \sigma^2/\lambda_i^2-e^{2t}})^{-1} \xi_i(t) = e^{-t} \frac{\sigma^2}{\lambda_i^2} \frac{\lambda_i}{\sigma^2} (U^\top y)_i, \\
\Rightarrow~ & \xi_i(t) = \frac{e^{t}}{\lambda_i(1 + \sigma^2/\lambda_i^2-e^{2t})}(U^\top y)_i.
\end{align}

\paragraph{Proof of \Cref{coro:boosted_posterior}} Given \Cref{thm:density_equivalence}, we only need to prove that sampling from
$$
    \nu_t = \mathsf{T}_{Q_t, b_t} \pi_t ~\propto~\pi_t(x)\exp\left(-\frac12 x^\top Q_t x + x^\top b_t\right)
$$
is equivalent to sampling from a new posterior. Taking $\pi_t$ as the corresponding prior, we only need to show that the factor $\exp\left(-\frac12 x^\top Q_t x + x^\top b_t\right)$ is the likelihood of certain observation model in the form of $\tilde{y} = A_t x + w$ with $w\sim \gamma_d$. To end, we need to ensure
$$
\exp\left(-\frac12 x^\top Q_t x + x^\top b_t\right) ~\propto ~\exp(-\frac12 \|A_t x - \tilde{y} \|^2)
$$
Choosing $A_t$ in the standard SVD form $A_t = \Sigma_t V^\top$ where the singular values of $A_t$ (the diagonal entries of $\Sigma_t$) are $\frac{e^{t}}{(1 + \sigma^2/\lambda_i^2 - e^{2t})^{1/2}}$ for $1 \leq i \leq d'$, the quadratic term is matched. Matching the first order term requires $b_t = A^\top\tilde{y} = V\Sigma_t^\top \tilde{y}$. Further matching the coefficients in the basis of $V$ requires that 
$$
(\Sigma_t^\top \tilde{y})_i=\xi_i(t)= \frac{e^{t}}{\lambda_i(1 + \sigma^2/\lambda_i^2-e^{2t})}(U^\top y)_i, \quad 1 \leq i \leq d'.
$$
It is easy to verify that $\tilde{y} = \Sigma'_t U^\top y$ with $\Sigma'_t = \text{diag}\left( \frac{1}{(\sigma^2 + \lambda_1^2(1 - e^{2t}))^{1/2}}, \ldots, \frac{1}{(\sigma^2 + \lambda_{d'}^2(1 - e^{2t}))^{1/2}} \right)$ satisfies the above requirement.

\begin{remark} Also, one can directly use the backward transport equation (\ref{eq:transport_backwards}) 
to generate samples with the probability flow ODE~\cite{song2020score} backward in time
\begin{equation}
    \dd X_t\reverse = (-X_t\reverse - \nabla \log \pi_{t}(X_t\reverse))\dd t
\end{equation}
from $X_T\reverse \sim \gamma_d$.
Combining the fact that $\nu_t$ satisfies the PDE~\eqref{eq:reverse_sde} and $\Delta \nu_t = \nabla \cdot (\nabla \nu_t) = \nabla \cdot (s_t - Q_t x + b_t) $, we have that $\nu_t$ also satisfies the transport equation 
$
\partial_t p_t = \nabla \cdot ((s_t+(I+Q_t)x-b_t) p_t).
$
Therefore, for any $t<T$, by initializing $X_t \sim \nu_t$ and run the reverse ODE 
$
\dd X_t\reverse = (-(I+Q_t)X\reverse_t -\nabla \log \pi_t(X\reverse_t) + b_t)\dd t
$
then $X_0\reverse$ also gives the desired posterior. However, we note that unlike the reverse SDE case, the corresponding transport PDE and the vector field in the reverse ODE case are different from those used in the prior data generation. As discussed below, both $Q_t$ and $b_t$ are singular near $T$. Therefore running the reverse SDE might be preferrable for better numerical stability.
\end{remark}

\subsection{Proof for the Heat Semigroup Setting}
\label{app:heat_proof}
We first follow the proof of \Cref{thm:density_equivalence} to derive the ODEs for tilted transport in the heat semigroup setting.
We again denote the time-dependent score function $\nabla \log \pi_t(x)$ by $s_t(x)$. As derived in~\Cref{sec:preliminary}, if we initialize $X_{\tau}$ according to density $\rho_{\tau}$ and run the reverse SDE~\cref{eq:reverse_sde_heat}, the density of $X_t$ for $t\leq \tau$, denoted by $\rho_t$, satisfies the backward PDE:
\begin{equation}
    \partial_t \rho_t = \nabla \cdot (2s_t \rho_t) - \Delta \rho_t.
\end{equation}
We again wish to find $Q_t, b_t$ such that 
$$
\nu_t := \mathsf{T}_{Q_t, b_t} \pi_t
$$ satisfies the above PDE.
Note that $\rho_t$ satisfies this PDE is equivalent to that $h_t=\log \rho_t$ satisfies the following Hamilton-Jacobi PDE
\begin{equation}
    \partial_t h_t = 2\nabla \cdot s_t + 2\nabla h_t \cdot s_t - (\Delta h_t + \|\nabla h_t \|^2).
\end{equation}
By definition, we know $h_t = \log \pi_t$ satisfies the above PDE, and we need to find $h_t = \log \nu_t = \log \pi_t - \frac12 x^\top Q_t x + x^\top b_t + F(t)$ satisfying this PDE as well. Here $F(t)$ denotes the normalizing constant. Taking the difference between two equations, we need
\begin{align}
    &-\frac12 x^\top \dot{Q}_t x + x^\top \dot{b}_t + \dot{F} = 2(-Q_tx+b_t)\cdot s_t + \text{trace}(Q_t) + \| s_t \|^2 - \| s_t - Q_t x + b_t \|^2 \\
    \Leftrightarrow~~
    &-\frac12 x^\top \dot{Q}_t x + x^\top \dot{b}_t + \dot{F} = -x^\top(Q_t^\top Q_t)x + 2x^\top (Q_t^\top b_t) + \|b_t\|^2 + \text{trace}(Q_t)
\end{align}
which leads us to the ODE~\eqref{eq:boost_ode_heat}. Therefore, the eigenvalues of $Q_t$ evolve according to the ODE $\dot{q} = 2q^2$, whose solution has the form
\begin{equation}
    \lambda_i(t) = \frac{\lambda_i}{1-2 t \lambda_i}.
\end{equation}
So the ODE will blowup at time $T=\|Q\|^{-1}/2$. The full solution of $Q_t$ and $b_t$ can be solved similarly as in~\Cref{app:ode_solution}.

\section{Proofs of Section \ref{sec:quant}}

\subsection{Proof of \Cref{thm:strong_logconcave_chi}}
\label{app:strong_logconcave_proof_chi}
\begin{proof}[Proof of \Cref{thm:strong_logconcave_chi}]
By definition, we need to show 
\begin{equation}
    -\lambda_{\min}(Q_{T}) + \sup_x \lambda_{\max}(\nabla^2 \log \pi_{T}(x) ) < 0.
\end{equation}
From the argument in the main text, we know $T = \frac12 \log(1+\lambda_{\max}(Q)^{-1})$, and thus
$$\lambda_{\min}(Q_{T}) = \frac{1+\lambda_{\max}(Q)^{-1}}{\lambda_{\min}(Q)^{-1}-\lambda_{\max}(Q)^{-1}}.
$$
From Corollary \ref{coro:hessianscore}, we have 
\begin{equation}
    \sup_x \lambda_{\max}( \nabla^2 \log \pi_{T}(x)) \leq 
     (1 +\|Q\|) \left( \|Q\| \chi_{\|Q\|}(\pi) - 1 \right)~.
\end{equation}

 Let $m=\lambda_{\min}(Q)$. Therefore, we can guarantee that $\nu_{T}$ is strongly log-concave if 
\begin{align}
\label{eq:basic}
    \frac{1+\|Q\|^{-1}}{m^{-1} - \|Q\|^{-1}} &>  (1 +\|Q\|) \left( \|Q\| \chi_{\|Q\|}(\pi) - 1 \right) ~,
\end{align}
or equivalently 
 \begin{align}
     \chi_{\|Q\|}(\pi) &< \|Q\|^{-1}\frac{\kappa^2(A) }{\kappa^2(A) - 1}~.
 \end{align}

\end{proof}

\begin{lemma}[Hessian of Gaussian Mixture Potential]
\label{lem:hessian_GMM}
    Let $\pi = \mu \ast \gamma_{\Sigma}$ be a Gaussian mixture.
    Then $\nabla^2 \log \pi(x) = \Sigma^{-1}\left( \mathrm{Cov}\left[\mathsf{T}_{\Sigma^{-1}, \Sigma^{-1}x } \mu \right] \Sigma^{-1}  -I \right)~.$  
\end{lemma}
\begin{proof}
    Let us first compute the score 
    $\nabla \log \pi$. By definition we have 
    \begin{align}
        \nabla \log \pi(x) &= -\Sigma^{-1}\left(x - \frac{\int y \mu(y) e^{-\frac12 (x-y)^\top \Sigma^{-1} (x-y)} dy}{\int \mu(y) e^{-\frac12 (x-y)^\top \Sigma^{-1} (x-y)} dy} \right) \\
        &= -\Sigma^{-1}( x - \mathbb{E}\left[ \mathsf{T}_{\Sigma^{-1}, \Sigma^{-1}x } \mu \right])~,
    \end{align}
    and thus 
    \begin{align}
        \nabla^2 \log \pi(x) &= \Sigma^{-1}\left(\mathrm{Cov}\left[\mathsf{T}_{\Sigma^{-1}, \Sigma^{-1}x } \mu \right] \Sigma^{-1} - I \right)~,
    \end{align}
    where we defined $\mathrm{Cov}[\mu] = \mathbb{E}_{x\sim \mu}[x x^\top] - (\mathbb{E}_{x \sim \mu} x)(\mathbb{E}_{x \sim \mu} x)^\top$. 
\end{proof}

\begin{corollary}
\label{coro:hessianscore}
In particular, we have 
\begin{equation}
    \sup_x \lambda_{\max}( \nabla^2 \log \pi_{T}(x)) \leq 
     (1 +\|Q\|) \left( \|Q\| \chi_{\|Q\|}(\pi) - 1 \right)~.
\end{equation}
\end{corollary}
\begin{proof}
   From Lemma \ref{lem:hessian_GMM} and $\pi_{t} = \mathsf{C}_{1 - e^{-2t}} ( \mathsf{D}_{e^{t}} \pi)=\mathsf{D}_{e^{t}} \pi \ast \gamma_{1-e^{-2t}}$, we directly have 
   \begin{equation}
       \nabla^2 \log \pi_t(x) = (1-e^{-2t})^{-1}\left( (1-e^{-2t})^{-1} \mathrm{Cov}\left[\mathsf{T}_{(1-e^{-2t})^{-1},(1-e^{-2t})^{-1}x }(\mathsf{D}_{e^t}\pi) \right] - I\right)~.
   \end{equation}
   Now, using the commutation property between the isotropic tilt and the dilation 
   $ \D_\alpha \T_{\eta, \theta} = \T_{\alpha^2 \eta , \alpha \theta} \D_\alpha$, we have
   \begin{align}
       \mathrm{Cov}\left[\mathsf{T}_{(1-e^{-2t})^{-1},(1-e^{-2t})^{-1}x }(\mathsf{D}_{e^t}\pi) \right] &= \mathrm{Cov}\left[ \mathsf{D}_{e^t} \mathsf{T}_{(e^{2t}-1)^{-1},e^{-t}(1-e^{-2t})^{-1}x  } \pi\right] \\
       &= e^{-2t} \mathrm{Cov}\left[ \mathsf{T}_{(e^{2t}-1)^{-1},e^{-t}(1-e^{-2t})^{-1}x  } \pi \right]~,
   \end{align}
   and therefore 
   \begin{align}
       \sup_x \left\|  \mathrm{Cov}\left[\mathsf{T}_{(1-e^{-2t})^{-1},(1-e^{-2t})^{-1}x }(\mathsf{D}_{e^t}\pi) \right] \right\| &\leq e^{-2t} \chi_{(e^{2t}-1)^{-1}}(\pi)~.
   \end{align}
 
   Using that $e^{2T} -1 = \|Q\|^{-1}$, we thus obtain 
   \begin{align}
       \sup_x \lambda_{\max}(\nabla^2 \log \pi_T(x) ) &\leq (1 +\|Q\|) \left( \|Q\| \chi_{\|Q\|}(\pi) - 1 \right)~. 
   \end{align}
\end{proof}

\begin{lemma}[Isotropic Tilt of a Gaussian Mixture]
\label{lemma:gaussian_mixt_tilt}
    If $\pi = \mu \ast \gamma_\delta$, then 
    \begin{equation}
        \mathsf{T}_{tI, z} \pi = \tilde{\mu} \ast \gamma_{\sigma^2} ~,
    \end{equation}
    where $\sigma^{-2} = \delta^{-1} + t$ and 
    \begin{equation}
        \tilde{\mu}( \tilde{y} ) \propto \mu( ( \sigma^{-2} \tilde{y} - z) \delta) e^{\frac12 \left(\sigma^{-2} \| \tilde{y} \|^2 - \delta \| \sigma^{-2} \tilde{y} - z\|^2 \right)}~. 
    \end{equation}
\end{lemma}
\begin{proof}
    By definition, we have
    $$\mathsf{T}_{tI, z} \pi \propto \int d\mu(y) e^{-\frac12 t \|x\|^2 + x \cdot z - \frac12 \delta^{-1} \|x-y\|^2}~. $$
By expressing 
$$-\frac12 t \|x\|^2 + x \cdot z - \frac12 \delta^{-1} \|x-y\|^2 = -\frac12 \sigma^{-2} \| x - \tilde{y} \|^2 + C$$
we have 
\begin{align}
    \sigma^{-2} &= \delta^{-1} + t ~,\\
    \tilde{y} & = \frac{\delta^{-1} y + z }{\delta^{-1} + t} ~,\\
    C & = \frac12 \left[ \sigma^{-2} \| \tilde{y}\|^2 - \delta^{-1} \|y\|^2 \right]~,
\end{align}
which gives the desired result after performing the affine change of variables from $y$ to $\tilde{y}$.    
\end{proof}

\subsection{Proof of Proposition \ref{prop:sufficient_cond}}
\label{app:proof_sufficient}
\begin{proof}

Applying \cite[Lemma 1]{ma2019sampling}, under our assumptions, we can 
approximate the potential $f$ (recall that $\pi=e^{-f}$) with a $m/2$-strongly-convex potential $\tilde{f}$, satisfying $\mathrm{osc}[ f - \tilde{f} ] \leq 16 L R^2$, where the oscillation of a function is defined as $\mathrm{osc}(g):= \sup_x g(x) - \inf_x g(x)$. 

Now, for any $x$, we approximate the potential of the quadratic tilt 
$$d\mathsf{T}_{t, tx} \pi (y) = e^{-\frac{t}{2} \| y - x\|^2 - f(y) - \log Z_x} \dd y $$ by 
$$\tilde{f}_x(y):= \tilde{f}(y) + \frac{t}{2} \| y - x\|^2 + \log Z_x.$$
We verify that $\tilde{f}_x$ is $(m/2 + t)$-strongly convex, and that 
$$\mathrm{osc}\left[\tilde{f}_x - \log \mathsf{T}_{t, tx} \pi \right] = \mathrm{osc}\left[\tilde{f} - f \right] \leq 16 L R^2~.$$

By the Holley-Strook perturbation principle and the Barky-Emery criterion, we therefore have that, under our assumptions, the measures $\mathsf{T}_{t, tx} \pi$ satisfy a Poincar\'e inequality \emph{uniformly} over $x$, with constant $C:=(m/2+t) e^{- 16 L R^2}$. 

Finally, observe
\begin{align}
    \|\mathrm{Cov}[\mathsf{T}_{t, tx}\pi]\| &= \sup_{\| \theta \|=1} \mathrm{Var}_{\mathsf{T}_{t,tx}\pi} [\theta^\top Y] \\
    &\leq C^{-1} \mathbb{E}_{\mathsf{T}_{t,tx}\pi} [\|\theta\|^2] = C^{-1}~,
\end{align}
showing that $\chi_t(\pi) \leq (m/2+t)^{-1} e^{16 L R^2}$, as claimed. 
\end{proof}

\subsection{Proof of Proposition \ref{prop:example_blowup}}
\label{app:proof_counter}
\begin{proof}
Let $\pi \in \mathcal{P}(\R)$ be of the form 
$$\pi = \sum_{j \geq 0} \alpha_j \delta_{y_j}~,\text{ with } y_j = \sum_{i \leq j} b_i~,$$
where $0 \leq b_1 \leq b_2 \leq \ldots$ is an non-decreasing sequence with $\lim_{j \to \infty} b_j = \infty$, and $\alpha_{2k} = \alpha_{2k+1}$, $b_{2k} = b_{2k+1}$ for all $k \in \mathbb{N}$, satisfying 
\begin{equation}
\label{eq:chi_example_cond}
    \sum_j \alpha_j = 1~, \quad \sum_j \alpha_j j^2 b_j^2 < \infty~.
\end{equation}
Note that $\mathbb{E}_\pi[ y^2] \leq \sum_j \alpha_j j^2 b_j^2$ and 
thus $\pi \in \mathcal{P}_2(\R)$. 

We set $x=\frac12(y_{2k} + y_{2k+1})$ and consider the tilted measure $\mathsf{T}_{t, tx} \pi$: it is also atomic and supported in $\{ y_j\}_j$, with weights
$$\tilde{\alpha}_j = \frac{\alpha_j e^{-\frac12 t(y_j - x)^2}}{\sum_i \alpha_i e^{-\frac12 t(y_i - x)^2}}~.$$
Let us now lower bound the weights $\tilde{\alpha}_{2k}$ and $\tilde{\alpha}_{2k+1}$. We have 
\begin{align}
    \sum_i \alpha_i e^{-\frac12 t(y_i - x)^2} &= (\alpha_{2k} + \alpha_{2k+1})e^{-\frac{t}{8} b_{2k+1}^2} + \sum_{i \neq 2k, 2k+1} \alpha_i e^{-\frac12 t(y_i - x)^2} \\
    &\leq  e^{-\frac{t}{8} b_{2k+1}^2} 2\alpha_{2k} + e^{-\frac{t}{2}(b_{2k+1}/2 + b_{2k})^2 } \\
    &= e^{-\frac{t}{8} b_{2k}^2} \left(2\alpha_{2k} + e^{-t b_{2k}^2}  \right)~,
\end{align}
thus 
\begin{align}
    \tilde{\alpha}_{2k} &\geq \frac{\alpha_{2k} e^{-\frac{t}{8} b_{2k}^2}}{e^{-\frac{t}{8} b_{2k}^2} \left(2\alpha_{2k} + e^{-t b_{2k}^2 } \right)} \\
    &= %
    \frac{1}{2 + \alpha_{2k}^{-1} e^{-t b_{2k}^2 }}~,
\end{align}
and clearly $\tilde{\alpha}_{2k+1} = \tilde{\alpha}_{2k} $. 

Now, observe that 
\begin{align}
    \mathrm{Var}(\mathsf{T}_{t, tx} \pi) &\geq \frac{\tilde{\alpha}_{2k} + \tilde{\alpha}_{2k+1}}{4} b_{2k}^2~. 
\end{align}
Therefore, the susceptibility $\chi_t(\pi)$ satisfies 
\begin{equation}
\label{eq:suscept_counter}
\chi_t(\pi) \geq \sup_k \frac{b_{2k}^2}{2(2 + \alpha_{2k}^{-1}e^{-t b_{2k}^2 })}~.
\end{equation}

We identify two regimes of blowup for $\chi_t(\pi)$: for any $t>0$ when $\pi$ has heavy tails, and for $t \geq \lambda > 0 $ for subgaussian tails. 

Indeed, for the first setting, let $\alpha_{2k} \simeq k^{-s}$, $b_{2k} \simeq k^{r}$, with $s>0$, $r>0$ and $s - 2r > 3$. Then \eqref{eq:chi_example_cond} can be satisfied, and we verify that for any $t>0$ the RHS of \eqref{eq:suscept_counter} is unbounded. 
Finally, by choosing $\lambda > 0$, setting $b_{2k}^2 = k$ and $\alpha_{2k} = e^{-\lambda k}$ for $k>k^*$ with a large enough $k^* \in \mathbb{N}$ and appropriate $\alpha_{2k}$ for $k\leq k^*$, we can have the condition~\eqref{eq:chi_example_cond} satisfied.
We verify that for this choice of parameters, $\pi$ has subgaussian tails; Indeed, for sufficiently large $z$, 
\begin{align}
    \mathbb{P}_\pi( Y \geq z) &\lesssim e^{-\lambda z^2}~.
\end{align}
Meanwhile, we have a blowup for the susceptibility as soon as $t > \lambda$:
\begin{equation}
\chi_t(\pi) \geq \sup_{k>k^*} \frac{k}{2(2 + e^{(\lambda-t)k})} = \infty~.
\end{equation}

\end{proof}

\subsection{Proof of Proposition \ref{ex:chiprop}}
\label{app:proof_chiprop}
\begin{proof}[Proof of Proposition \ref{ex:chiprop}]
 If $\mu$ is a product measure, we observe that the isotropic tilt $\mathsf{T}_t \mu$ is also a product measure, and therefore its covariance is diagonal. 

If $\pi = e^f$ with $\nabla f$ $L$-Lipschitz, we apply the Brascamp-Lieb inequality to $ \mathsf{T}_{t, tx} \pi$:
\begin{theorem}[Brascamp-Lieb Inequality, \cite{brascamp1976extensions}] If $\pi$ is a strongly-log-concave measure on $\R^d$, ie of the form $\pi = e^{-f}$ with $\nabla^2 f(x) \succeq \alpha \mathrm{Id}$ for all $x \in \R^d$, then $\| \mathrm{Cov}(\pi) \| \leq \alpha^{-1}$.   
\label{thm:brascamp-lieb}
\end{theorem}

 Observe that $ \mathsf{T}_{t, tx} \pi (y) \propto e^{-t/2 \|y\|^2 + t x.y + f}$, and thus 
$$\nabla^2 \left(-\log \mathsf{T}_{t, tx} \pi\right) \succeq (t - L)\mathrm{Id}~,$$
and hence $\chi_t(\pi) \leq \frac{1}{t-L}$ provided $t > L$.

\end{proof}

\begin{proof}[Proof of Example \ref{ex:chiexamples}]
    The first example is immediate, after observing that $\mathsf{T}_t \gamma$ is a Gaussian of variance $(1+t)^{-1}$. 
 For the Gaussian mixture example, we observe from \Cref{lemma:gaussian_mixt_tilt} that 
 $\mathsf{T}_t ( \mu \ast \gamma_\delta)$ is a Gaussian mixture of variance $(t + \delta^{-1})^{-1}$, where the mixture distribution is supported in a ball of radius $R \frac{\delta^{-1}}{\delta^{-1} + t}$. 
 Moreover, the covariance of a homogeneous mixture of the form $ \mu \ast \gamma_{\Sigma}$ is $\Sigma + \mathrm{Cov}(\mu)$. 

 Finally, by the \Cref{ex:chiprop}, if $\pi$ is the uniform measure on the hypercube, then 
 $\chi_t(\pi) = \chi_t( \frac12 (\delta_{-1} + \delta_{+1})) = 1$.

\end{proof}

\subsection{Proof of \Cref{coro:gaussmixt}}
\label{app:gaussmixt}

\begin{proof}[Proof of \Cref{coro:gaussmixt}]
We plug the susceptibility $\chi_t(\pi) = \chi_t( \mu \ast \gamma_\delta) \leq  \left( \frac{R}{1+\delta t}\right)^2 + \frac{\delta}{1+\delta t}$ from Example \ref{ex:chiexamples} (ii) into \cref{eq:necessary_cond} to get (we use $\kappa$ to denote $\kappa(A)$ for simplicity) 
\begin{align}
&\|Q\|^{-1} \frac{\kappa^2 }{(\kappa^2 - 1)} > \left( \frac{R}{1+\delta \|Q\|}\right)^2 + \frac{\delta}{1+ \delta \|Q\| } \\
\Leftrightarrow~ & (1 + \delta \|Q\|) \left(\frac{((1 + \delta \|Q\|))\kappa^2}{\|Q\|(\kappa^2-1)} -\delta \right) > R^2 \\
\Leftrightarrow~ & \frac{((1 + \delta \|Q\|)(\kappa^2+\delta \|Q\|)}{\|Q\|(\kappa^2-1)} > R^2 \\
\Leftrightarrow~ & \frac{(1+\delta\SNR^2)(\delta\kappa^2+\SNR^{-2})}{\kappa^2-1} > R^2~.
\end{align}
\end{proof}

\subsection{Importance Sampling}
\label{app:is_exponential}

\begin{proof}[Proof of \Cref{prop:ISlower}]
We wish to lower bound $\mathrm{KL}(\nu || \pi)$ with tools of functional inequalities for the concentration of measure. 
By the celebrated work~\cite{otto2000generalization,bobkov2001hypercontractivity} that the log-Sobolev inequality implies the Talagrand
transport-entropy inequality, we have
\begin{align}
    \mathrm{KL}(\nu || \pi) \geq \frac{\mathrm{LSI}(\nu) W(\nu, \pi)^2}{2}.
\label{eq:T2}
\end{align}
Here $W(\nu, \pi)$ denotes the Wassertain distance between $\nu$ and $\pi$:
$$
W(\nu, \pi) = \sqrt{\inf_{\gamma \in \Gamma(\nu, \pi)} \|x-y\|^2 \dd \gamma(x,y)}~,
$$
where $\Gamma(\nu, \pi)$ denotes the set of probability measures on $\R^d \times \R^d$ with marginals $\nu$ and $\pi$. 

First by \Cref{eq:ISass1}, we have
\begin{equation}
\label{eq:hessian_lb}
    \nabla^2 \left(-\log \nu \right) \succeq (\SNR - L)\mathrm{Id}~,
\end{equation}
which gives
\begin{equation}
    \mathrm{LSI}(\nu) \geq (\SNR - L)~,
\label{eq:LSI_lb}
\end{equation}
by Bakry-Emery criterion~\cite{bakry2006diffusions}.
Furthermore, we have the lower bound for the Wasserstain distance~\cite{gelbrich1990formula}
\begin{align}
    W^2(\nu, \pi)  \geq &~ \| \text{mean}(\nu) - \text{mean}(\pi) \|^2 + \text{trace}(\text{Cov}(\nu) + \text{Cov}(\pi) - 2(\text{Cov}(\pi)^{\frac12}\text{Cov}(\nu)\text{Cov}(\pi)^{\frac12})^{\frac12}) \\
    \geq &~ \text{trace}(\text{Cov}(\nu) + I_d - 2\text{Cov}(\nu)^{\frac12})\\
    = &~  \text{trace}(\text{Diag}(\text{Cov}(\nu)) + I_d - 2\text{Diag}(\text{Cov}(\nu))^{\frac12}) \\
    = &~  \sum_{i=1}^d((1 - \text{Std}(\nu)_i)^2)~,
\end{align}
where the second-to-last equality comes from the fact that the trace remains unchanged under orthogonal transformation.
By \eqref{eq:hessian_lb} and Brascamp-Lieb Inequality~(\Cref{thm:brascamp-lieb}), we have 
$\text{Std}(\nu)_i \leq \sqrt{\frac{1}{\SNR - L}}$.
With the condition $\SNR-L > 2$,
\begin{align}
W^2(\nu, \pi) \geq \left(1-\sqrt{\frac12}\right)^2d~.
\label{eq:W2_lb}
\end{align}
Collecting \cref{eq:LSI_lb,eq:T2,eq:W2_lb}, we obtain our final estimate
\begin{align}
    \mathrm{KL}(\nu||\pi) \geq \mathcal{O}(d\cdot \SNR)~.
\end{align}
\end{proof}

\subsection{Stability Analysis}
\label{app:stability_proof}

The proof is similar to that in~\cite{song2021maximum}. Here we provide the proof in our posterior sampling context for completeness.

\begin{proof}[Proof of \Cref{prop:stability}]
Consider the following two reverse dynamics needed for the error estimate: one is based on the exact score and starts from the exact boosted posterior
\begin{equation}
    \dd X_\tau = (-X_\tau - 2\nabla \log \pi_{\tau}(X_\tau))\dd \tau + \sqrt{2}\dd {W}_\tau,~~X_t \sim \nu_t,
\label{eq:reverse_sde1}
\end{equation}
and another one is based on the approximate score and starts from the approximation to the boosted posterior
\begin{equation}
        \dd \tilde{X}_\tau = (-\tilde{X}_\tau - 2 s_{\theta}(\tilde{X}_\tau, \tau))\dd \tau + \sqrt{2}\dd {W}_\tau,~~\tilde{X}_t \sim q_t,
\label{eq:reverse_sde2}
\end{equation}
Note that these two dynamics are defined backwardly for $\tau\in [0, t]$ and we drop the superscript $\reverse$ for notation simplicity. We denote the path measure of $\{X_\tau\}_{\tau \in[0,t]}$ and $\{\tilde{X}_\tau\}_{\tau \in[0,t]}$ by $\bm{\nu}$ and $\bm{q}$, respectively. Then $\nu$ and $q_0$ are the marginal distributions of the two path measures at $t=0$. By data processing inequality and chain rule of KL divergence, we have
\begin{align}
    \mathrm{KL}(\nu || q_0) &\leq \mathrm{KL}(\bm{\nu} || \bm{q}) \\
    &\leq \mathrm{KL}(\nu_{\tau} || q_{\tau}) + \E_{z\sim \nu_{\tau}}\mathrm{KL}(\bm{\nu}(\cdot | X_{t}=z)||\bm{q}(\cdot | \tilde{X}_{t}=z))
\end{align}
Given the Novikov’s condition, we can apply the Girsanov theorem~\cite{oksendal2013stochastic} to \cref{eq:reverse_sde1}\ref{eq:reverse_sde2} to compute the second term above
\begin{align}
&~\E_{z\sim \nu_{\tau}}\mathrm{KL}(\bm{\nu}(\cdot | X_{t}=z)||\bm{q}(\cdot | \tilde{X}_{t}=z)) \\
\leq&~-\E_{\bm{\nu}}\left[\log \frac{\dd \bm{q}}{\dd \bm{\nu}}\right] \\
=&~\E_{\bm{\nu}}\left[2\int_0^t(\nabla \log {\pi_\tau}(x) - s_{\theta}(x, \tau)) \dd W_\tau + \int_0^t \|\nabla \log {\pi_\tau}(x) - s_{\theta}(x, \tau)\|^2 \dd \tau\right] \\
=&~\E_{\bm{\nu}}\left[\int_0^t \|\nabla \log {\pi_\tau}(x) - s_{\theta}(x, \tau)\|^2 \dd \tau\right] \\
=&~\int_{0}^t \E_{\nu_\tau} \| \nabla \log {\pi_\tau}(x) - s_{\theta}(x, \tau)\|^2 \dd \tau~.
\end{align}

\end{proof}

\section{Experimental Details}
\label{app:exp_details}
\subsection{Gaussian Mixture Models}
\label{app:gmm_results}
For a given dimension $d$ with $d ~\text{mod}~2 = 0$, we consider prior data a mixture of 25 Gaussian distributions, the same as considered in \cite{cardoso2024monte}. The Gaussian distribution has mean $(8i, 8j, \cdots, 8i, 8j) \in  \R^d$ for $(i, j) \in \{-2, -1, 0, 1, 2\}^2$ and unit variance. Each (unnormalized) mixture weight is independently drawn according to a $\chi^2$ distribution.

For the measurement model considered in \Cref{fig:gmm_err}, we generate $A$ in the following way. We first sample a $d\times d$ matrix with each entry sampled from the standard normal and compute its SVD to get $U$ and $V$ for $A$. The singular value is given by $[1, \cdots, 1/20]$ where each component in between is independently sampled from $\mathrm{Unif}([1/20, 1])$ such that the condition number of $A$ is 20. The observation noise is then determined by SNR. For the measurement model considered in \Cref{tab:gmm}, the matrix $U$ and $V$ for the SVD form of $A$ is the same to the above. Each singular value in $S$ is independently sampled from $\mathrm{Unif}([0, 1])$, and $\sigma$ is sampled from $\mathrm{Unif}([0.2\max{S}, \max{S}])$.

For all the experiments we run the boosted posterior from $T-0.01$ such that the ODE solution $Q, b$ is well-defined. We use BlackJAX~\cite{cabezas2024blackjax} to implement the No-U-turn sampler.

Besides results reported in~\Cref{fig:gmm_err}, we further test tilted transport when $d' < d$. In this setting, $\lambda_{\min}(Q_t)$ remains zero but the signal corresponding to the non-zero eigenvalues still gets enhanced. Therefore, although it becomes more difficult for $\nu_{T}$ to be strongly log-concave, the tilted transport can still make the new posterior easier to sample even if it is not strongly log-concave yet. As shown in \Cref{tab:gmm}, when $d' = 0.9d$, 10\% percent eigenvalues of $Q_t$ are zero, our tilted transport technique still reduces the statistical distance of the posterior samples significantly. We also consider an even more challenging case where $d' = 1$ such that the target posterior is still heavlily multimodal (as visualized in the 2D example in \Cref{fig:phase_diagram}). In this case, LMC suffers from the local maxima of the potential and thus cannot explore the multimodal distribution efficiently. We use the No-U-turn sampler\cite{hoffman2014no}, a Hamilton Monte Carlo (HMC) method, as the baseline method, which can move among different modes more efficiently than Langevin. We find that the tilted transport technique can still boost the performance of HMC in this challenging setting. We also verify \Cref{thm:density_equivalence} by sampling from the boosted posterior directly from its analytical formula and running the reverse SDE, and the obtained samples approximate the target posterior well, as reported in \Cref{tab:gmm}.

\begin{table}[!ht]
\setlength{\tabcolsep}{4pt}
\caption{Sliced Wasserstein distance for Gaussian mixture prior for degenerate case}
\label{tab:gmm}
\centering
\begin{tabular}{c||ccc|ccc}
\toprule
\multirow{5}{*}{$d$}  & \multicolumn{3}{c|}{$d'=0.9d$} & \multicolumn{3}{c}{$d'=1$}\\
  \cmidrule(lr){2-4} \cmidrule(lr){5-7}
   & Langevin          & \makecell{Boosted\\Langevin}
   & \makecell{Analytic\\Boost}   & HMC         & \makecell{Boosted\\ HMC}  & \makecell{Analytic\\Boost} \\ \midrule
20  & $4.21 \pm1.87$ &  $\bm{2.32 \pm2.42}$  &  $\mathit{0.02 \pm 0.00}$ & $1.33 \pm1.02$ &  $\bm{1.11 \pm0.83}$ & $\mathit{0.12 \pm 0.07}$\\
40  & $4.09 \pm2.02$ &  $\bm{2.45 \pm1.79}$  &  $\mathit{0.02 \pm 0.00}$ & $2.04 \pm1.26$ &  $\bm{1.81 \pm1.03}$ & $\mathit{0.13 \pm 0.07}$\\
80  & $4.40 \pm2.31$ &  $\bm{2.75 \pm2.10}$  &  $\mathit{0.02 \pm 0.00}$ & $2.98 \pm2.15$ &  $\bm{2.77 \pm2.32}$ & $\mathit{0.11 \pm 0.06}$\\
\bottomrule
\end{tabular}
\end{table}

\subsection{Imaging Problems}
\label{app:image_results}
We perform four inverse tasks on the Flickr-Faces-HQ Dataset (FFHQ)~\cite{karras2019style} to demonstrate the application of the tilted transport technique on imaging data as a proof of concept. To apply the proposed tilted transport technique to these problems, we still need to select a baseline method for sampling from the boosted posterior $\nu_{T}$. In the case of ill-conditioned problems, sampling $\nu_{T}$ may still be challenging for principled algorithms like LMC, and we still need to rely on heuristic methods for imaging tasks. However, as noted in the introduction, most existing heuristic methods primarily facilitate conditional generation based on the measurement, lacking principled guarantees for posterior sampling. Consequently, we lack a principled interpretation for enhancing these methods with tilted transport. Nevertheless, we can still experiment with such methods as a proof of concept. We chose Diffusion Model Based Posterior Sampling (DMPS)~\cite{meng2022diffusion} as the baseline method for the following reasons: The main assumption of DMPS in approximating the time-dependent conditional score is that the prior $\pi$ is uninformative (flat) with respect to $X_t$, such that $p(X_0|X_t) \propto p(X_t|X_0)$.  This assumption only holds approximately in early phases of the forward diffusion, and hopefully a higher SNR provided by tilted transport makes the effect of this approximation error smaller.

We conducted four tasks: (a) denoising; (b) inpainting with random masks from~\cite{saharia2022palette}; (c) $4\times$ super-resolution; and (d) deblurring using a Gaussian kernel. Our algorithm was implemented using the NVIDIA codebase~\cite{mardani2023variational} with 1000 diffusion steps for posterior sampling, and utilized the score function from a pretrained diffusion model~\cite{choi2021ilvr}. Similar to our Gaussian mixture model experiments where we adjusted the timing for the boosted posterior to avoid the singularity of $Q_t$, we shifted 6 - 10 timesteps for setting the boosted posterior. Experiments show that the final performance is robust with respect to the number of shifted steps. \Cref{fig:inpaint} showcases examples from the inpainting task, demonstrating how tilted transport enhances the baseline DMPS method. Additionally, we report various sample statistics including peak signal-to-noise ratio (PSNR), structural similarity index (SSIM), and Learned Perceptual Patch Similarity (LPIPS). However, it is important to note that while these statistics assess the quality of prior data generation, they may not accurately reflect the quality of posterior samples.

\begin{figure}[!ht]
\centering
\includegraphics[width=1\textwidth]{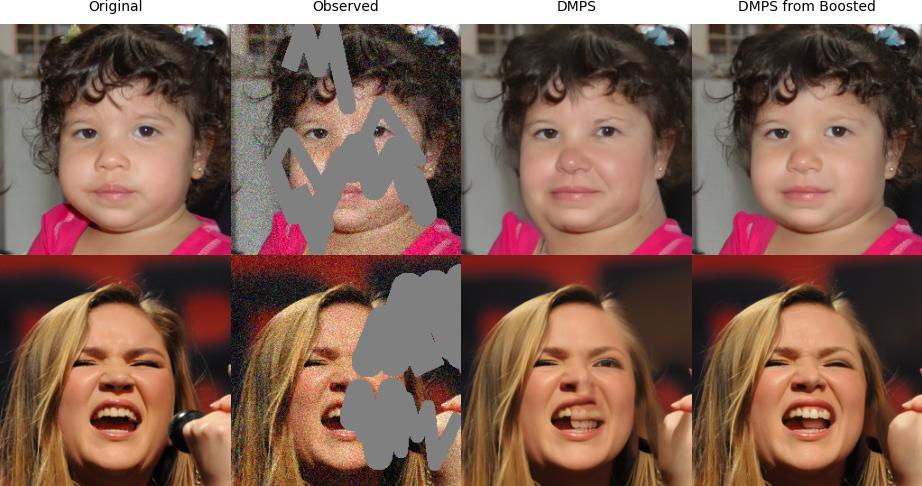}
\caption{Examples for inpainting with random masks over FFHQ dataset}
\label{fig:inpaint}
\end{figure}

\begin{table}[!ht]
    \centering
    \renewcommand\tabcolsep{5pt} %
    \caption{Performance for tasks on FFHQ Dataset.}
    \begin{tabular}{l||cc|cc|cc|cc}
    \toprule
    \textbf{Task} & \multicolumn{2}{c|}{\textbf{Denoising}} & \multicolumn{2}{c|}{\textbf{Inpainting}} & \multicolumn{2}{c|}{\textbf{Super-resolution}}& \multicolumn{2}{c}{\textbf{Deblur}}\\
    \cmidrule(r){1-1} \cmidrule(lr){2-3} \cmidrule(lr){4-5} \cmidrule(lr){6-7} \cmidrule(lr){8-9}
    Metrics &DMPS & Boost & DMPS & Boost & DMPS & Boost & DMPS & Boost\\
    \midrule
    PSNR(dB) $\uparrow$ &32.153 &\textbf{32.350} & 22.458 & \textbf{23.312} & 26.761 & \textbf{26.899} & 29.088 & \textbf{29.155}\\
    SSIM $\uparrow$ &0.886 &0.886 & 0.786 & \textbf{0.800} & 0.760 & \textbf{0.754} &0.815 & 0.815\\
    LPIPS $\downarrow$ &0.060 &\textbf{0.039} & 0.131 & \textbf{0.098} & 0.129 & \textbf{0.109} & 0.098 & \textbf{0.094}\\
    \bottomrule
    \end{tabular}
    \label{tab:ffhq}
\end{table}

%% file: main_arxiv.bbl
\newcommand{\etalchar}[1]{$^{#1}$}
\begin{thebibliography}{SDWMG15}

\bibitem[AJK{\etalchar{+}}21]{anari2021entropic}
Nima Anari, Vishesh Jain, Frederic Koehler, Huy~Tuan Pham, and Thuy-Duong Vuong.
\newblock Entropic independence {I}: Modified log-{S}obolev inequalities for fractionally log-concave distributions and high-temperature ising models.
\newblock {\em arXiv preprint arXiv:2106.04105}, 2021.

\bibitem[And82]{anderson1982reverse}
Brian~DO Anderson.
\newblock Reverse-time diffusion equation models.
\newblock {\em Stochastic Processes and their Applications}, 12(3):313--326, 1982.

\bibitem[BB19]{bauerschmidt2019very}
Roland Bauerschmidt and Thierry Bodineau.
\newblock A very simple proof of the {LSI} for high temperature spin systems.
\newblock {\em Journal of Functional Analysis}, 276(8):2582--2588, 2019.

\bibitem[BBD23]{bauerschmidt2023stochastic}
Roland Bauerschmidt, Thierry Bodineau, and Benoit Dagallier.
\newblock Stochastic dynamics and the {P}olchinski equation: an introduction.
\newblock {\em arXiv preprint arXiv:2307.07619}, 2023.

\bibitem[BD22]{bauerschmidt2022log}
Roland Bauerschmidt and Benoit Dagallier.
\newblock Log-{S}obolev inequality for the {$\varphi^{4}_2$} and {$\varphi^{4}_3$} measures.
\newblock {\em arXiv preprint arXiv:2202.02295}, 2022.

\bibitem[BDBDD23]{benton2023linear}
Joe Benton, Valentin De~Bortoli, Arnaud Doucet, and George Deligiannidis.
\newblock Linear convergence bounds for diffusion models via stochastic localization.
\newblock {\em arXiv preprint arXiv:2308.03686}, 2023.

\bibitem[BDBDD24]{benton2024nearly}
Joe Benton, Valentin De~Bortoli, Arnaud Doucet, and George Deligiannidis.
\newblock Nearly d-linear convergence bounds for diffusion models via stochastic localization.
\newblock In {\em The Twelfth International Conference on Learning Representations}, 2024.

\bibitem[B{\'E}85]{bakry2006diffusions}
Dominique Bakry and Michel {\'E}mery.
\newblock Diffusions hypercontractives.
\newblock In Jacques Az{\'e}ma and Marc Yor, editors, {\em S{\'e}minaire de Probabilit{\'e}s XIX 1983/84}, pages 177--206. Springer Berlin Heidelberg, Berlin, Heidelberg, 1985.

\bibitem[BGL01]{bobkov2001hypercontractivity}
Sergey~G Bobkov, Ivan Gentil, and Michel Ledoux.
\newblock Hypercontractivity of {H}amilton--{J}acobi equations.
\newblock {\em Journal de Math{\'e}matiques Pures et Appliqu{\'e}es}, 80(7):669--696, 2001.

\bibitem[BGL14]{bakry2014analysis}
Dominique Bakry, Ivan Gentil, and Michel Ledoux.
\newblock {\em Analysis and geometry of Markov diffusion operators}, volume 103.
\newblock Springer, 2014.

\bibitem[BGPS17]{barnett2017rapid}
Alex Barnett, Leslie Greengard, Andras Pataki, and Marina Spivak.
\newblock Rapid solution of the cryo-{EM} reconstruction problem by frequency marching.
\newblock {\em SIAM Journal on Imaging Sciences}, 10(3):1170--1195, 2017.

\bibitem[BHK{\etalchar{+}}19]{barak2019nearly}
Boaz Barak, Samuel Hopkins, Jonathan Kelner, Pravesh~K Kothari, Ankur Moitra, and Aaron Potechin.
\newblock A nearly tight sum-of-squares lower bound for the planted clique problem.
\newblock {\em SIAM Journal on Computing}, 48(2):687--735, 2019.

\bibitem[BL76]{brascamp1976extensions}
Herm~Jan Brascamp and Elliott~H Lieb.
\newblock On extensions of the {B}runn-{M}inkowski and {P}r{\'e}kopa-{L}eindler theorems, including inequalities for log concave functions, and with an application to the diffusion equation.
\newblock {\em Journal of functional analysis}, 22(4):366--389, 1976.

\bibitem[BLB08]{bickel2008sharp}
Peter Bickel, Bo~Li, and Thomas Bengtsson.
\newblock Sharp failure rates for the bootstrap particle filter in high dimensions.
\newblock In {\em Pushing the limits of contemporary statistics: Contributions in honor of Jayanta K. Ghosh}, volume~3, pages 318--330. Institute of Mathematical Statistics, 2008.

\bibitem[BXY23]{brennecke2023operator}
Christian Brennecke, Changji Xu, and Horng-Tzer Yau.
\newblock Operator norm bounds on the correlation matrix of the {SK} model at high temperature.
\newblock {\em arXiv preprint arXiv:2307.12535}, 2023.

\bibitem[CCL{\etalchar{+}}23]{chen2023sampling}
Sitan Chen, Sinho Chewi, Jerry Li, Yuanzhi Li, Adil Salim, and Anru Zhang.
\newblock Sampling is as easy as learning the score: theory for diffusion models with minimal data assumptions.
\newblock In {\em The Eleventh International Conference on Learning Representations}, 2023.

\bibitem[CCL{\etalchar{+}}24]{chen2024probability}
Sitan Chen, Sinho Chewi, Holden Lee, Yuanzhi Li, Jianfeng Lu, and Adil Salim.
\newblock The probability flow {ODE} is provably fast.
\newblock {\em Advances in Neural Information Processing Systems}, 36, 2024.

\bibitem[CCLL24]{cabezas2024blackjax}
Alberto Cabezas, Adrien Corenflos, Junpeng Lao, and Rémi Louf.
\newblock Blackjax: Composable {B}ayesian inference in {JAX}, 2024.

\bibitem[CD18]{chatterjee2018sample}
Sourav Chatterjee and Persi Diaconis.
\newblock The sample size required in importance sampling.
\newblock {\em The Annals of Applied Probability}, 28(2):1099--1135, 2018.

\bibitem[CE22]{chen2022localization}
Yuansi Chen and Ronen Eldan.
\newblock Localization schemes: A framework for proving mixing bounds for markov chains.
\newblock In {\em 2022 IEEE 63rd Annual Symposium on Foundations of Computer Science (FOCS)}, pages 110--122. IEEE, 2022.

\bibitem[Cel24]{celentano2024sudakov}
Michael Celentano.
\newblock {S}udakov--{F}ernique post-{AMP}, and a new proof of the local convexity of the {TAP} free energy.
\newblock {\em The Annals of Probability}, 52(3):923--954, 2024.

\bibitem[Che97]{chen1997inverse}
Yu~Chen.
\newblock Inverse scattering via {H}eisenberg's uncertainty principle.
\newblock {\em Inverse problems}, 13(2):253, 1997.

\bibitem[CJCM24]{cardoso2024monte}
Gabriel Cardoso, Yazid {Janati El idrissi}, Sylvain~Le Corff, and Eric Moulines.
\newblock Monte {C}arlo guided denoising diffusion models for {B}aysian linear inverse problems.
\newblock In {\em The Twelfth International Conference on Learning Representations}, 2024.

\bibitem[CKJ{\etalchar{+}}21]{choi2021ilvr}
Jooyoung Choi, Sungwon Kim, Yonghyun Jeong, Youngjune Gwon, and Sungroh Yoon.
\newblock {ILVR}: Conditioning method for denoising diffusion probabilistic models.
\newblock {\em arXiv preprint arXiv:2108.02938}, 2021.

\bibitem[CKM{\etalchar{+}}23]{chung2023diffusion}
Hyungjin Chung, Jeongsol Kim, Michael~Thompson Mccann, Marc~Louis Klasky, and Jong~Chul Ye.
\newblock Diffusion posterior sampling for general noisy inverse problems.
\newblock In {\em The Eleventh International Conference on Learning Representations}, 2023.

\bibitem[CKS24]{chen2024learning}
Sitan Chen, Vasilis Kontonis, and Kulin Shah.
\newblock Learning general gaussian mixtures with efficient score matching.
\newblock {\em arXiv preprint arXiv:2404.18893}, 2024.

\bibitem[CLL23]{chen2023improved}
Hongrui Chen, Holden Lee, and Jianfeng Lu.
\newblock Improved analysis of score-based generative modeling: User-friendly bounds under minimal smoothness assumptions.
\newblock In {\em International Conference on Machine Learning}, pages 4735--4763. PMLR, 2023.

\bibitem[CSY22]{chung2022come}
Hyungjin Chung, Byeongsu Sim, and Jong~Chul Ye.
\newblock Come-closer-diffuse-faster: Accelerating conditional diffusion models for inverse problems through stochastic contraction.
\newblock In {\em Proceedings of the IEEE/CVF Conference on Computer Vision and Pattern Recognition}, pages 12413--12422, 2022.

\bibitem[DS24]{dou2024diffusion}
Zehao Dou and Yang Song.
\newblock Diffusion posterior sampling for linear inverse problem solving: A filtering perspective.
\newblock In {\em The Twelfth International Conference on Learning Representations}, 2024.

\bibitem[EAG24]{el2024bounds}
Ahmed El~Alaoui and Jason Gaitonde.
\newblock Bounds on the covariance matrix of the {S}herrington--{K}irkpatrick model.
\newblock {\em Electronic Communications in Probability}, 29:1--13, 2024.

\bibitem[EAMS22]{el2022sampling}
Ahmed El~Alaoui, Andrea Montanari, and Mark Sellke.
\newblock Sampling from the {S}herrington-{K}irkpatrick {G}ibbs measure via algorithmic stochastic localization.
\newblock In {\em 2022 IEEE 63rd Annual Symposium on Foundations of Computer Science (FOCS)}, pages 323--334. IEEE, 2022.

\bibitem[EKZ22]{eldan2022spectral}
Ronen Eldan, Frederic Koehler, and Ofer Zeitouni.
\newblock A spectral condition for spectral gap: fast mixing in high-temperature ising models.
\newblock {\em Probability theory and related fields}, 182(3):1035--1051, 2022.

\bibitem[Eld20]{eldan2020taming}
Ronen Eldan.
\newblock Taming correlations through entropy-efficient measure decompositions with applications to mean-field approximation.
\newblock {\em Probability Theory and Related Fields}, 176(3):737--755, 2020.

\bibitem[FSR{\etalchar{+}}23]{feng2023score}
Berthy~T Feng, Jamie Smith, Michael Rubinstein, Huiwen Chang, Katherine~L Bouman, and William~T Freeman.
\newblock Score-based diffusion models as principled priors for inverse imaging.
\newblock In {\em Proceedings of the IEEE/CVF International Conference on Computer Vision}, pages 10520--10531, 2023.

\bibitem[Gel90]{gelbrich1990formula}
Matthias Gelbrich.
\newblock On a formula for the {L}2 wasserstein metric between measures on euclidean and {H}ilbert spaces.
\newblock {\em Mathematische Nachrichten}, 147(1):185--203, 1990.

\bibitem[GJP{\etalchar{+}}24]{gupta2024diffusion}
Shivam Gupta, Ajil Jalal, Aditya Parulekar, Eric Price, and Zhiyang Xun.
\newblock Diffusion posterior sampling is computationally intractable.
\newblock {\em arXiv preprint arXiv:2402.12727}, 2024.

\bibitem[GMJS22]{graikos2022diffusion}
Alexandros Graikos, Nikolay Malkin, Nebojsa Jojic, and Dimitris Samaras.
\newblock Diffusion models as plug-and-play priors.
\newblock {\em Advances in Neural Information Processing Systems}, 35:14715--14728, 2022.

\bibitem[Gro75]{gross1975logarithmic}
Leonard Gross.
\newblock Logarithmic {S}obolev inequalities.
\newblock {\em American Journal of Mathematics}, 97(4):1061--1083, 1975.

\bibitem[G{\v{S}}V16]{galanis2016inapproximability}
Andreas Galanis, Daniel {\v{S}}tefankovi{\v{c}}, and Eric Vigoda.
\newblock Inapproximability of the partition function for the antiferromagnetic ising and hard-core models.
\newblock {\em Combinatorics, Probability and Computing}, 25(4):500--559, 2016.

\bibitem[HDMOB24]{heurtel2024listening}
David Heurtel-Depeiges, Charles~C Margossian, Ruben Ohana, and Bruno R{\'e}galdo-Saint Blancard.
\newblock Listening to the noise: Blind denoising with gibbs diffusion.
\newblock {\em arXiv preprint arXiv:2402.19455}, 2024.

\bibitem[Her56]{herbert1956empirical}
Robbins Herbert.
\newblock An empirical bayes approach to statistics.
\newblock In {\em Proceedings of the third berkeley symposium on mathematical statistics and probability}, volume~1, pages 157--163, 1956.

\bibitem[HG14]{hoffman2014no}
Matthew~D Hoffman and Andrew Gelman.
\newblock The {N}o-{U}-{T}urn sampler: adaptively setting path lengths in {H}amiltonian {M}onte {C}arlo.
\newblock {\em Journal of Machine Learning Research}, 15(1):1593--1623, 2014.

\bibitem[HJA20]{ho2020denoising}
Jonathan Ho, Ajay Jain, and Pieter Abbeel.
\newblock Denoising diffusion probabilistic models.
\newblock {\em Advances in neural information processing systems}, 33:6840--6851, 2020.

\bibitem[HP86]{haussmann1986time}
Ulrich~G Haussmann and Etienne Pardoux.
\newblock Time reversal of diffusions.
\newblock {\em The Annals of Probability}, pages 1188--1205, 1986.

\bibitem[HS87]{holley1986logarithmic}
Richard Holley and Daniel Stroock.
\newblock Logarithmic {S}obolev inequalities and stochastic ising models.
\newblock {\em Journal of Statistical Physics}, 46(5–6):1159–1194, March 1987.

\bibitem[JAD{\etalchar{+}}21]{jalal2021robust}
Ajil Jalal, Marius Arvinte, Giannis Daras, Eric Price, Alexandros~G Dimakis, and Jon Tamir.
\newblock Robust compressed sensing mri with deep generative priors.
\newblock {\em Advances in Neural Information Processing Systems}, 34:14938--14954, 2021.

\bibitem[JDMO24]{janati2024divide}
Yazid Janati, Alain Durmus, Eric Moulines, and Jimmy Olsson.
\newblock Divide-and-conquer posterior sampling for denoising diffusion priors.
\newblock {\em arXiv preprint arXiv:2403.11407}, 2024.

\bibitem[JKO98]{jordan1998variational}
Richard Jordan, David Kinderlehrer, and Felix Otto.
\newblock The variational formulation of the fokker--planck equation.
\newblock {\em SIAM journal on mathematical analysis}, 29(1):1--17, 1998.

\bibitem[KEES22]{kawar2022denoising}
Bahjat Kawar, Michael Elad, Stefano Ermon, and Jiaming Song.
\newblock Denoising diffusion restoration models.
\newblock {\em Advances in Neural Information Processing Systems}, 35:23593--23606, 2022.

\bibitem[KLA19]{karras2019style}
Tero Karras, Samuli Laine, and Timo Aila.
\newblock A style-based generator architecture for generative adversarial networks.
\newblock In {\em Proceedings of the IEEE/CVF conference on computer vision and pattern recognition}, pages 4401--4410, 2019.

\bibitem[Kun23]{doi:10.1137/1.9781611977912.180}
Dmitriy Kunisky.
\newblock Optimality of {G}lauber dynamics for general-purpose {I}sing model sampling and free energy approximation.
\newblock {\em Proceedings of the 2024 Annual ACM-SIAM Symposium on Discrete Algorithms (SODA)}, pages 5013--5028, 2023.

\bibitem[KVE21]{kawar2021snips}
Bahjat Kawar, Gregory Vaksman, and Michael Elad.
\newblock {SNIPS}: Solving noisy inverse problems stochastically.
\newblock {\em Advances in Neural Information Processing Systems}, 34:21757--21769, 2021.

\bibitem[KWB19]{kunisky2019notes}
Dmitriy Kunisky, Alexander~S Wein, and Afonso~S Bandeira.
\newblock Notes on computational hardness of hypothesis testing: Predictions using the low-degree likelihood ratio.
\newblock In {\em ISAAC Congress (International Society for Analysis, its Applications and Computation)}, pages 1--50. Springer, 2019.

\bibitem[Led01]{ledoux2001concentration}
M.~Ledoux.
\newblock {\em The concentration of measure phenomenon}.
\newblock Mathematical surveys and monographs. American Mathematical Society, 2001.

\bibitem[LLT22]{lee2022convergence}
Holden Lee, Jianfeng Lu, and Yixin Tan.
\newblock Convergence for score-based generative modeling with polynomial complexity.
\newblock {\em Advances in Neural Information Processing Systems}, 35:22870--22882, 2022.

\bibitem[LLT23]{lee2023convergence}
Holden Lee, Jianfeng Lu, and Yixin Tan.
\newblock Convergence of score-based generative modeling for general data distributions.
\newblock In {\em International Conference on Algorithmic Learning Theory}, pages 946--985. PMLR, 2023.

\bibitem[MCF15]{ma2015complete}
Yi-An Ma, Tianqi Chen, and Emily Fox.
\newblock A complete recipe for stochastic gradient {MCMC}.
\newblock {\em Advances in neural information processing systems}, 28, 2015.

\bibitem[MCJ{\etalchar{+}}19]{ma2019sampling}
Yi-An Ma, Yuansi Chen, Chi Jin, Nicolas Flammarion, and Michael~I Jordan.
\newblock Sampling can be faster than optimization.
\newblock {\em Proceedings of the National Academy of Sciences}, 116(42):20881--20885, 2019.

\bibitem[MK22]{meng2022diffusion}
Xiangming Meng and Yoshiyuki Kabashima.
\newblock Diffusion model based posterior sampling for noisy linear inverse problems.
\newblock {\em arXiv preprint arXiv:2211.12343}, 2022.

\bibitem[MSKV23]{mardani2023variational}
Morteza Mardani, Jiaming Song, Jan Kautz, and Arash Vahdat.
\newblock A variational perspective on solving inverse problems with diffusion models.
\newblock {\em arXiv preprint arXiv:2305.04391}, 2023.

\bibitem[MV00]{markowich2000trend}
Peter~A Markowich and C{\'e}dric Villani.
\newblock On the trend to equilibrium for the {F}okker-{P}lanck equation: an interplay between physics and functional analysis.
\newblock {\em Mat. Contemp}, 19:1--29, 2000.

\bibitem[Oks13]{oksendal2013stochastic}
Bernt Oksendal.
\newblock {\em Stochastic differential equations: an introduction with applications}.
\newblock Springer Science \& Business Media, 2013.

\bibitem[OV00]{otto2000generalization}
Felix Otto and C{\'e}dric Villani.
\newblock Generalization of an inequality by talagrand and links with the logarithmic {S}obolev inequality.
\newblock {\em Journal of Functional Analysis}, 173(2):361--400, 2000.

\bibitem[RT96]{roberts1996exponential}
Gareth~O Roberts and Richard~L Tweedie.
\newblock Exponential convergence of langevin distributions and their discrete approximations.
\newblock {\em Bernoulli}, pages 341--363, 1996.

\bibitem[SCC{\etalchar{+}}22]{saharia2022palette}
Chitwan Saharia, William Chan, Huiwen Chang, Chris Lee, Jonathan Ho, Tim Salimans, David Fleet, and Mohammad Norouzi.
\newblock Palette: Image-to-image diffusion models.
\newblock In {\em ACM SIGGRAPH 2022 conference proceedings}, pages 1--10, 2022.

\bibitem[SDME21]{song2021maximum}
Yang Song, Conor Durkan, Iain Murray, and Stefano Ermon.
\newblock Maximum likelihood training of score-based diffusion models.
\newblock {\em Advances in neural information processing systems}, 34:1415--1428, 2021.

\bibitem[SDWMG15]{sohl2015deep}
Jascha Sohl-Dickstein, Eric Weiss, Niru Maheswaranathan, and Surya Ganguli.
\newblock Deep unsupervised learning using nonequilibrium thermodynamics.
\newblock In {\em International conference on machine learning}, pages 2256--2265. PMLR, 2015.

\bibitem[SLK22]{shoushtari2022dolph}
Shirin Shoushtari, Jiaming Liu, and Ulugbek~S Kamilov.
\newblock Dolph: Diffusion models for phase retrieval.
\newblock {\em arXiv preprint arXiv:2211.00529}, 2022.

\bibitem[SSDK{\etalchar{+}}20]{song2020score}
Yang Song, Jascha Sohl-Dickstein, Diederik~P Kingma, Abhishek Kumar, Stefano Ermon, and Ben Poole.
\newblock Score-based generative modeling through stochastic differential equations.
\newblock {\em arXiv preprint arXiv:2011.13456}, 2020.

\bibitem[SSXE22]{song2022solving}
Yang Song, Liyue Shen, Lei Xing, and Stefano Ermon.
\newblock Solving inverse problems in medical imaging with score-based generative models.
\newblock In {\em International Conference on Learning Representations}, 2022.

\bibitem[SVMK22]{song2022pseudoinverse}
Jiaming Song, Arash Vahdat, Morteza Mardani, and Jan Kautz.
\newblock Pseudoinverse-guided diffusion models for inverse problems.
\newblock In {\em International Conference on Learning Representations}, 2022.

\bibitem[SZY{\etalchar{+}}23]{song2023loss}
Jiaming Song, Qinsheng Zhang, Hongxu Yin, Morteza Mardani, Ming-Yu Liu, Jan Kautz, Yongxin Chen, and Arash Vahdat.
\newblock Loss-guided diffusion models for plug-and-play controllable generation.
\newblock In {\em International Conference on Machine Learning}, pages 32483--32498. PMLR, 2023.

\bibitem[WTN{\etalchar{+}}23]{wu2024practical}
Luhuan Wu, Brian Trippe, Christian Naesseth, David Blei, and John~P Cunningham.
\newblock Practical and asymptotically exact conditional sampling in diffusion models.
\newblock In {\em Advances in Neural Information Processing Systems}, volume~36, pages 31372--31403, 2023.

\end{thebibliography}
